\documentclass[a4paper,11pt]{amsart}%
\usepackage{hyperref}
\usepackage{pdfsync}
\usepackage{dsfont}
\usepackage{color}
\usepackage{esint}
\usepackage{amsthm}
\usepackage{enumerate}
\usepackage{amsfonts}
\usepackage{amssymb}%
\usepackage{mathtools}
\usepackage{braket}
\usepackage{bm}
\usepackage{amsmath}
\usepackage{graphicx}
\usepackage{caption}
\usepackage{subcaption}
\usepackage{fullpage}

\usepackage{mathtools}
\numberwithin{equation}{section}
\usepackage{amsfonts}
\usepackage{graphicx}%
\providecommand{\U}[1]{\protect\rule{.1in}{.1in}}
\newtheorem{theorem}{Theorem}[section]

\newtheorem{lemma}[theorem]{Lemma}

\newtheorem{proposition}[theorem]{Proposition}
\newtheorem{remark}[theorem]{Remark}

\newcommand{\veps}{\varepsilon}
\newcommand{\e}{\varepsilon}

\newcommand{\x}{x}
\newcommand{\R}{\mathbb{R}}

\renewcommand{\L}{\Delta}
\DeclareMathOperator{\Prob}{\mathbb{P}}
\newcommand{\M}{\mathcal{M}}
\newcommand{\bgamma}{\bm{\gamma}}

\definecolor{mygreen}{rgb}{0.1,0.75,0.2}

\newcommand{\nc}{\normalcolor}
\newcommand{\E}{\mathbb{E}}

\newcommand{\Var}{Var}
\newcommand{\eps}{\varepsilon}

\newcommand{\N}{\mathbb{N}}

\title{ A maximum principle argument for the uniform convergence of graph Laplacian regressors}
\author{Nicol\'as Garc\'ia Trillos}
\address{Department of Statistics, University of Wisconsin, Madison, Wisconsin, USA}
\email{garciatrillo@wisc.edu}
\author{Ryan Murray}
\address{Department of Mathematics, North Carolina State University, Raleigh, NC, USA}
\email{rwmurray@ncsu.edu}



\begin{document}
	\thanks{*To be published in SIMODS}
\maketitle

\begin{abstract}
%\blue
This paper investigates the use of methods from partial differential equations and the Calculus of variations to study learning problems that are regularized using graph Laplacians. Graph Laplacians are a powerful, flexible method for capturing local and global geometry in many classes of learning problems, and the techniques developed in this paper help to broaden the methodology of studying such problems. In particular, we develop the use of maximum principle arguments to establish asymptotic consistency guarantees within the context of noise corrupted, non-parametric regression with samples living on an unknown manifold embedded in $\R^d$. The maximum principle arguments provide a new technical tool which informs parameter selection by giving concrete error estimates in terms of various regularization parameters. A review of learning algorithms which utilize graph Laplacians, as well as previous developments in the use of differential equation and variational techniques to study those algorithms, is given. In addition, new connections are drawn between Laplacian methods and other machine learning techniques, such as kernel regression and k-nearest neighbor methods. 
\nc

% We study asymptotic consistency guarantees for a non-parametric regression problem with Laplacian regularization. \nc  In particular, we consider $(\x_1, y_1), \dots, (\x_n, y_n)$ samples from some distribution on the cross product $\M \times \R$, where $\M$ is an $m$-dimensional manifold embedded in $\R^d$. A geometric graph on the cloud $\{\x_1, \dots, \x_n \}$ is constructed by connecting points that are within some specified distance $\veps_n$. A suitable semi-linear equation involving the resulting graph Laplacian is used to obtain a regressor for the observed values of $y$ which are assumed to be noisy versions of some underlying trend. We establish probabilistic error rates for the uniform difference between the regressor constructed from the observed data and the Bayes regressor (or trend) associated to the ground-truth distribution. We give the explicit dependence of the rates in terms of the parameter $\veps$, the strength of regularization $\beta$, and the number of data points $n$. Our argument relies on a simple, yet powerful, maximum principle for the graph Laplacian. As a by product of our analysis we provide quantifiable scalings for how the strength of regularization should scale with the number of data points to avoid overfitting.
\end{abstract}

\section{Introduction}\label{sec:Intro}

In this paper we present new theoretical results on the consistency of solutions to a family of variational problems that use graph Laplacian regularization for trend filtering and supervised learning  with noisy labels, and determine scaling limits under which one can provably avoid overfitting as the number of data points grows. In addition, we draw new parallels between the methods studied here and other non-parametric regression methodologies found in the literature. Throughout the paper we highlight the analytical approach that we take in order to establish our high probability quantitative results, and that in particular allows us to study the behavior of solutions to  non-linear graph-based equations. This analytical approach provides a new avenue for studying algorithms that utilize graph Laplacians; such algorithms are utilized in a wide variety of learning problems (see Section \ref{sec:GraphPDEs} for further discussion)  \nc

Given a data set $X=\{x_1, \dots, x_n\}$ and corresponding \textit{noisy} labels $y_1, \dots, y_n$, the idea in trend filtering is to reconstruct a trend function $u$ taking inputs $x$ into outputs $y$ which closely matches the observed labels. Without further constraints, finding such a function is an ill-posed problem as there are many functions that will respect the observed data, most of which should be intuitively discarded as they overfit the observations $y_i$. To overcome this overfitting issue and to turn an ill-posed problem into a well-posed one, a popular idea used in applied mathematics \cite{tikhonov1963regularization} and statistics \cite{wahba1990spline,hoerl1970ridge} is to introduce a functional $R$ (a ``regularizer'') which penalizes ``irregular" functions, and to solve an optimization problem of the form
\[ \min  R(u) +  F(u,y), \]
where $F(u,y)$ represents a loss function measuring empirical risk. In a classical statistical setting the loss function is dictated by the noise model, but in general, the function $F(u,y)$ may simply be interpreted as a mismatch function between the observations and the regressor $u$. On the other hand, while \nc the choice of regularizer is largely open ended, intuitively it should be selected to enforce ``smoothness" with respect to some underlying geometry (informally, one wants to force the solution to not change too much by making small perturbations of the input).  \nc

For generic set of points in Euclidean space, a popular choice of regularizer is the squared norm of a reproducing kernel Hilbert space (RKHS) as discussed in \cite{Belkin-Niyogi-Sindhwani}. This approach is analogous to the use of ridge regression in classical statistics, and in fact can actually be rigorously cast in this way after transforming the data through a map canonically induced by the reproducing property of a kernel (see \cite{Smola-Kernel-Book}).

As elegant and convenient as the use of RKHSs may be, a generic choice of kernel to be used within the above regularization procedure will typically be oblivious to the specific geometric configuration of a data set $X$. This is particularly problematic in settings like that of semi-supervised learning, where one hopes the underlying geometry of the data set reveals information that the limited number of available labels can not. Motivated by this discussion, several authors,  including \cite{Lafferty1,smola2003kernels,Belkin-Niyogi-Sindhwani,Ando}, consider regularizers which exploit the intrinsic geometric structure of a data set. In their set-up, a data set $X$ comes with an additional weighted graph structure $\Gamma =(X, W)$, where $W$ represents an $n \times n$ similarity matrix and which intuitively captures the geometry of $X$.  
%In all the previously mentioned works the graph is explicitly used for the construction of regularizers enforcing
%solutions of the optimization problems to be ``smooth" with respect to the geometry summarized by the graph. 

In this paper we study intrinsic regularization in the context of trend filtering and fully supervised learning. As in \cite{smola2003kernels} the optimization problems we are interested in take the form
\begin{equation}
 \min_{u} J(u),\qquad J(u) := \beta R_\Gamma(u) + \frac{1}{n}\sum_{i=1}^nF(u(x_i) - y_i),\nc
\label{op:graphBelkin} 
\end{equation}
where we have made the dependence of the regularizer $R$ on the graph $\Gamma$ explicit. The empirical risk that we consider here is directly linked to the log-likelihood of an \textit{assumed} distribution $e^{-F(-s)}$ for the noise terms in an additive model of the form
\begin{equation}
y_i= \mu(x_i) + \xi_i  
\label{eqn:assumedModel}
\end{equation}
where $\mu$ is some underlying ground-truth function. Notice that the \textit{assumed} noise distribution may not be the same as the \textit{actual} noise distribution which we will denote as $p(s)ds$. The standard choice for $F$ is the square loss corresponding to a Gaussian model, and we will give particular attention to this choice later on. On the other hand, the regularizer $R_{\Gamma}$ that we will focus on \nc in this paper is the graph \textit{Dirichlet energy} defined by
\begin{equation}
R_{\Gamma}(u):= \frac{1}{2n}\sum_{i,j}w_{ij}| u(x_i)- u(x_j)|^2.
\label{eqn:GraphDirichlet}
\end{equation}
The relevance of the graph Dirichlet energy is due to its close connection to the linear in $u$ \textit{graph Laplacian} 
\begin{equation}
\label{eqn:graphLap}
\Delta_\Gamma u (x_i) := \sum_{j=1}^n w_{ij} ( u(x_i) - u(x_j)) , \quad x_i \in X,
\end{equation}
indeed, it is straightforward to show that 
\[ R_\Gamma(u)= \langle \Delta_\Gamma u ,u \rangle_{X}. \]\nc
The graph Dirichlet energy is a special choice of intrinsic regularizer with several properties that will be discussed throughout the paper. Other sensible intrinsic regularizers are the graph $p$-Sobolev norms (obtained by changing the square in the above energy with a $p$-th power) for some $p \geq 1$ (and in particular when $p >m$ where $m$ is the intrinsic dimension of the data set; see \cite{Lp-laplacian-El-Alaoui-2016}).  The practical use of Laplacian regularization (or similar variants) was proposed some years ago \cite{smola2003kernels}, but has recently received new attention in the statistics and machine learning community \cite{rosasco2013nonparametric,tibshirani2014adaptive,wang2016trend}.  \nc  Bayesian approaches to learning where graph Laplacians are used as covariance matrices in order to define Gaussian priors that exploit the intrinsic geometry of the data set have been considered in \cite{Lafferty1,kirichenko2017,BertoStuart}. We notice that from a Bayesian perspective, the solution of problem \eqref{op:graphBelkin} is the MAP estimator (maximum a posteriori estimator) for $u$, in a model where the unknown variable $u$ is assumed to have a Gaussian prior distribution with mean zero and covariance matrix equal to $\beta$ times the graph Laplacian, and where the observations $y$ depend on $u$ according to an assumed additive noise model of the form \eqref{eqn:assumedModel} with noise distribution $e^{-F(-s)}$.  \nc

Going back to our problem \eqref{op:graphBelkin}, we notice that the graph Laplacian appears explicitly in the optimality conditions satisfied by the minimizer of \eqref{op:graphBelkin}, namely
\begin{equation}
 \beta \Delta_\Gamma u  +   f(u-y) =  0,
\label{graphPDE1}
\end{equation}
where $f = F'$. Equation \eqref{graphPDE1} can be interpreted as an elliptic graph partial differential equation (PDE), which is linear in its highest order term, i.e. the term that involves ``derivatives" of $u$ (in this case the graph Laplacian $\Delta_\Gamma u$), but that overall is not linear unless the loss function $F$ is quadratic (so that its derivative is linear). The most appropriate terminology for such an equation is \textit{semilinear equation}. This graph PDE, which emerges in a fully supervised learning setting, is nothing but one of many examples in the family of graph PDE-based machine learning methodologies. Indeed, there is a large family of machine learning methodologies for supervised and unsupervised learning that, at their core, are described by a graph PDE involving the graph Laplacian, the fractional graph Laplacian (i.e. powers of the graph Laplacian), or the $p$-graph Laplacian, followed by an extension step (in supervised settings) or a clustering step (e.g. $k$-means after the embedding step in spectral clustering). Many of these approaches will be mentioned in section \ref{sec:GraphPDEs} below. This paper aims at studying the statistical properties of solutions to \eqref{graphPDE1}, with an emphasis on the analytical techniques that allow us to handle the nonlinear term in \eqref{graphPDE1}. More concretely, we aim at answering the following questions:
\nc
\begin{enumerate}
	\item What is the behavior of the solution to \eqref{op:graphBelkin} (alternatively solutions to the graph PDE \eqref{graphPDE1}) as the number of data points goes to infinity, when the data set is assumed to be a collection of i.i.d. samples from a distribution supported on a $m$-dimensional curved manifold $\M$ (here we intuitively think of $m \ll d$, although this will not be important for our analysis), the graph is obtained from the data set by giving high weights only to points that are within Euclidean distance $\veps$ of each other, and the labels are \textit{noisy} versions of a hidden trend function $\mu$?  
	\item How should $\beta$ scale with $n$ so that the solution to \eqref{op:graphBelkin} converges to the Bayes regressor (i.e. the underlying label trend) in the large data limit $n \rightarrow \infty$, and in particular so that the regression procedure does not overfit the data?
\end{enumerate}
 
 To answer these questions we exploit properties of the graph Laplacian, and in particular a \textit{maximum principle} argument at the graph level to prove that the solution of \eqref{graphPDE1} converges \textit{uniformly} towards the solution of an analogous homogenized partial differential equation on $\M$ with probability one (homogenized in $x$ and $y$). We provide explicit rates of convergence with high probability in terms of the number of samples $n$, the parameter $\veps$ controlling data connectivity in the graph, and the parameter $\beta$ controlling the strength of regularization; our rates turn out to be essentially minimax optimal (see the discussion at the end of section \ref{sec:litera}) \nc. The proposed maximum principle (see Proposition \ref{prop:maxpple} below) allows us to handle any sufficiently smooth, strongly convex  \nc $F$. We also provide a characterization for how $\beta$ must scale with $n$ in order to recover, in the large data limit, a modified trend $\mu_f$ which depends on $\mu$ and on the function $f$. Indeed, unless the function $F$ is quadratic (i.e. $f$ is linear), in the regime $n \rightarrow \infty$, $\beta:= \beta_n\rightarrow 0$ the trend $\mu$ may not be recovered, unless further assumptions on the distribution of the noisy labels are imposed. We provide uniform rates of convergence towards the modified trend. Stated in another way, we provide quantified asymptotic consistency estimates for this class of non-parametric regression algorithms. Our estimates will be decomposed into sample error (analogous to variance) and approximation estimates (analogous to bias). Our arguments to bound the sample error are essentially as difficult for the quadratic $F$ as for the general $F$ (and the maximum principle and monotonicity of $f=F'$ will be key to show this). For the approximation part, we will separate the analysis for the quadratic case from the general case, as a richer set of tools is  available in the quadratic case. \nc
 
 We would like to highlight that our results are novel with respect to the existing literature in several regards that we summarize below.\nc
\begin{itemize}
	\item We show that the solution to \eqref{op:graphBelkin} (alternatively \eqref{graphPDE1}) converges towards the solution of an analogous limiting variational problem. This contrasts with several results in the literature that establish a connection of the graph Laplacian and the negative Laplace-Beltrami operator $\Delta_\M$ from the point of view of \textit{pointwise convergence}, where one fixes a regular enough function $h : \M \rightarrow \R$ and establishes convergence rates for
	\[ \Delta_\M h - \Delta_\Gamma h.\]
      Some of these results include \cite{BelkinNiyogi,Singer,HeinMatthias2005FGtM,Coifman7426,gine}. We remark that such results do not allow one to immediately deduce convergence of solutions to optimization problems on graphs. In this paper we show how one can bootstrap the pointwise consistency estimates to obtain convergence of solutions to optimization problems on graphs that involve the graph Laplacian, using ideas from PDE theory. In particular, the regularity of the solution to a limiting PDE will be important when using the pointwise consistency of the graph Laplacian. In addition, a very careful construction of upper and lower \textit{barrier functions} for the solution of \eqref{graphPDE1}, and the use of a maximum principle at the graph level will be crucial tools in order to handle the randomness in the data points and their noisy labels. For more details see Section \ref{sec:MainThms}.

	\item While there are many recent results studying the consistency of solutions to optimization problems on graphs towards solutions to variational problems and PDEs on $\M$ (see Section \ref{sec:literature} below for a discussion), most of these results do not obtain rates of convergence. Here we use the extra structure of the graph Laplacian to obtain rates for the uniform convergence of solutions to \eqref{graphPDE1}. We emphasize that we obtain \textit{uniform} (i.e. $L^\infty$) rates of convergence instead of $L^2$-type or other type of averaged error, as it is most commonly found in the literature.  We will compare our work with the few exceptions to this general absence of quantitative estimates, which are usually obtained with different assumptions from ours, particularly with regards to the next point.
	
	 \item  In our set-up the labels $y_i$ are assumed to be noisy versions of a hidden trend function $\mu$ so that
	 \[ y_i = \mu(x_i) + \xi_i.\] 
	 We show that provided that the regularization parameter $\beta$ in \eqref{graphPDE1} does not decay too fast to zero (and we characterize this precisely), the solution to \eqref{graphPDE1} converges to some (modified) trend function in the limit. In particular, we show that graph Laplacian regularization is indeed capable of removing noise and prevents overfitting. 
	 
	 \item We connect the Laplacian regularization stemming from \eqref{graphPDE1} with other non-parametric methodologies for regression based on local averaging, and which adapt to the local geometry of the underlying manifold (e.g. $k$-NN regressors). This discussion is presented in Section \ref{sec:litera} below. 
%	 We would like to highlight however, that our proof method exploits the structure of graph Laplacians together with the regularity of the solution of a continuum PDE, and essentially avoids the use of any notion from classical statistical learning theory (e.g. VC theory) in order to produce the desired probabilistic estimates.
%	 
	 \item We emphasize that our approach allows us to handle loss functions $F$ that are different from quadratic. As mentioned earlier, from a statistical point of view this gives us some flexibility in the assumptions made on the observations. From a technical point of view, the effect that different $F$s have in the resulting graph PDE is the difference between having a linear PDE (in case $F$ is quadratic) and having a non-linear PDE (if $F$ is not quadratic) where in principle the analysis is more difficult.      \nc
\end{itemize}

We notice that the observed data $(x_i,y_i)$ affects the equation \eqref{graphPDE1} in two different ways: first, the $x_i$ determine the graph Laplacian and in this way the ``differential'' operator that we study is random, and second the noisy label observations $y_i$ affect the right hand side of the equation. At a high level \eqref{graphPDE1} can be thought as a numerical scheme for solving the limiting continuum PDE that we derive. However, the convergence analysis of such a \nc numerical scheme is not standard in the literature, given the randomness in \textit{both} the approximating Laplacian (in our case the graph Laplacian) and the noisy right hand side.  It is precisely the presence of the regularizer that removes noise in the limit and helps to prevent overfitting. In summary, when going from \eqref{graphPDE1} to the limiting  PDE, there are two parts in the equation that get simultaneously homogenized: on the one hand the graph Laplacian behaves like a Laplacian operator on $\M$, and on the other, the noisy labels homogenize to some trend function.

\nc
%
%Of course many methods are available for such non-parametric regression problems, e.g. kNN regressors \cite{Kpotufe}
%
%The ideas behind our proofs are strongly grounded in PDE theory. This continues a growing body of work (which we briefly review in Section \ref{sec:literature}), which draws on ideas from the calculus of variations, PDE theory, optimal transport, and in general mathematical analysis, in order to effectively study unsupervised, semi-supervised, and fully supervised learning algorithms on geometric graphs.  We believe that the ideas presented in this work are amenable to use in other classification and regression problems on graphs.
%
We now outline the remainder of the work. In Section \ref{sec:setup-results} we give a precise description of the problem and main results. In Section \ref{sec:comparison} we review related literature, comparing our results and techniques with other work. In Section \ref{sec:proof-of-main-results} we then present the proofs of our results.

\nc

\section{Problem set-up and Main results} \label{sec:setup-results}

\subsection{Set-up}
\label{sec:Setup}
We consider an $m$-dimensional smooth, compact, manifold $\M$ embedded in $\R^d$ with no boundary. A significant body of work studies how to estimate $m$ given a family of points (see e.g. \cite{LITTLE2017504}); for the purpose of this work we will treat $m$ as a priori known. Throughout the paper we will denote by $|x-\tilde{x}|$ the Euclidean distance in $\R^d$ and by $d_\M(x, \tilde{x})$ the geodesic distance of two points in $\M$. We will denote by $i_0$ the injectivity radius of the manifold $\M$. 
We recall that the injectivity radius of a manifold is defined as the maximum radius for which the exponential map $\exp_x : B_m(0, i_0) \subseteq \mathcal{T}_x \M \rightarrow B_\M(x,i_0)$ defines a diffeomorphism for all $x \in \M$ (see \cite{docarmo1992riemannian}). In the remainder we use $B_m$ to denote balls in $\R^m$, $B_\M$ for balls in $\M$ with the geodesic distance, and finally $B$ for balls in $\R^d$. Finally, we will denote by $dvol_\M$ the volume form of $\M$ and, after rescaling as necessary, we will assume that the volume of $\M$ is equal to one.

%We use $K$ to represent a uniform bound on the absolute value of the sectional curvature of $\M$ and $R$ to represent its reach. 

In the remainder, and unless otherwise stated, we will assume that $x_1, \dots, x_n$ are i.i.d. samples from the uniform distribution on $\M$. We let $\mu: \M \rightarrow \R$ be an unknown trend function and assume that the labels $y_1, \dots, y_n$ are obtained by 
\begin{equation}
y_i = \mu(x_i) + \xi_i,
\label{eqn:noisemodel}
\end{equation}
where the $\xi_i$ are identically distributed with distribution $p(s) ds$, are independent of each other and from the $x_i$, and satisfy
\[  \E(\xi_i )=0 , \quad   | \xi_i| \leq \sigma, \]
where $\sigma $ is a fixed constant specifying the noise level. We have chosen to focus on the bounded noise case in order to simplify some of the technical arguments, in particular the construction and bounds of the barrier functions $y^-$ and $y^+$ in the proof of Theorem \ref{thm:var}. We believe that many of the arguments and results could be modified to accommodate other noise models but we do not pursue it here. \nc
 
We assume that the loss function $F: \R \rightarrow [0,\infty)$ is a smooth function which satisfies $F(0)=0$ and $F''(x) > 0$ for all $x$; in some references this is called \emph{strong convexity}. In particular, this implies that $f:= F'$ is a strictly monotone function, and that $f'(x)>0$ for all $x > 0$ and $f'(x) < 0$ for all $x < 0$. \nc% is a strictly convex function, so that, in particular, its derivative $f:= F'$ is strictly monotone (i.e. $f'>0$).

All the theorems presented in Section \ref{sec:MainThms} are stated in the setting described above, but in Remark \ref{rem:non-uniform-rho} we write precisely how they should be restated to cover a more general setting. No difficulties arise when extending these results, other than having to deal with more cumbersome notation and longer expressions.

By way of notation, $\|h\|_{C^k} = \sum_{i=1}^k \|D^ih\|_{C^0}$, where $C^0$ is the space of continuous functions equipped with the supremum norm, and where when write the $\|D^ih\|_{C^0}$ we mean that we are summing the norm of all of the mixed derivatives of order $i$. For $0 < \alpha < 1$, space of H\"older continuous function $C^\alpha$ are continuous functions equipped with the norm
\[
\|h\|_{C^\alpha} = \sup_{x,y \in \M} \frac{|h(x)-h(y)|}{|x-y|^\alpha} + \|h\|_{C^0}.
\]
The space $C^{k,\alpha}$ is given by the norm $\|h\|_{C^{k,\alpha}} = \|h\|_{C^k} + \|D^k h\|_{C^{\alpha}}$. For any natural number $r$, the function space $H^r(\M)$ is the defined by the norm $\|h\|_{H^r} = \|h\|_{L^2(\M)} + \|D^r h\|_{L^2(\M)}$.
\nc

%
%\begin{remark}
%	\label{rem:generalsetting}
%A more general model for the data $(x_i, y_i)$ is described as follows. We assume that $(x_i, y_i)$ are samples from a joint distribution $\bgamma$ with the following regularity assumptions. By the disintegration theorem there exists a family $\{ \bgamma_x\}_{x \in \M}$ of probability measures on $\R$  (the conditional distributions of the second coordinate given the first one) and a distribution $\nu$ on $\M$ (the marginal of the first coordinate) such that
%\[\bgamma(A\times D) = \int_{A} \bgamma_x(D)   d \nu(x)    \]
%holds for every Borel subset $A$ of $\M$ and every Borel subset $D$ of $\R$. We assume that the measure $\nu$ has a smooth density $\rho$ with respect to $\M$'s volume form $d vol_{\M}$, and the conditional distribution $\gamma_x$ is assumed to have a smooth density $p_x(\cdot)$ with respect to the Lebesgue measure in $\R$. In addition we assume that the densities $p_x$ vary smoothly in $x \in \M$. The trend function $\mu: \M\rightarrow \R$ associated to the distribution $\bgamma$ is defined as 
%\[ \mu(x):=  \E \left(  y  \lvert \x=x \right) = \int_{\R} y d p_x(y) .\]
%
%\blue This remark isn't the easiest to read, maybe polish \nc
%\end{remark}

\nc
%In the remainder we assume that  $(\x_1, y_1), \dots, (\x_q, y_q)$ are i.i.d. samples from $\bgamma$ and that $\x_{q+1}, \dots, \x_n$ are independently drawn from $\nu$. \nc 
%
%

\subsubsection{Graph construction and graph PDE}
\label{sec:graph}
Given $X := \{\x_i\}_{i=1}^n$ we construct a geometric graph as follows. We let $\eta:[0,\infty) \to [0,\infty)$ be a non-increasing function which is only non-zero on $[0,1]$. We further assume $\eta$ to be Lipschitz continuous and normalized so that
\[
  \int_{\R^m} \eta(|z|)\,dz = 1.
\]
We define the constant
\begin{equation}
  \tau_\eta := \int_{\R^m} |z_1|^2 \eta(|z|)\,dz,
\label{sigmaeta}
\end{equation}
where $z_1$ is the first coordinate of $z$. \nc
Between every two vertices $\x_i, \x_j \in X$ we assign the weight
\begin{equation}\label{eqn:weight-def}
  w_{i,j} = \frac{2}{\tau_\eta \veps^{m+2}n}\eta\left( \frac{|\x_i - \x_j|}{\veps}\right).
\end{equation}
Here the $\e^m$ in the denominator is a rescaling so that the weights at each vertex sum to approximately one, while the $\e^2$ in the denominator is chosen so that $\frac{|u(x_i)-u(x_j)|^2}{\e^2} \sim |\nabla u|^2$. The weighted graph $(X,w)$ is a geometric graph representing the proximity of the sample points $\x_i$ in $\R^d$. We have rescaled the weights for convenience (in taking limits as $n \to \infty$).\nc

%
%Next, for any functions $u:X \to \R$, i.e. a function of the vertices of our graph, we define the \emph{discrete Dirichlet energy} $R_n(u)$ by
%\begin{equation}
%  R_n(u) := \frac{1}{n}\sum_{i=1}^n \sum_{j=1}^n w_{ij} (u(\x_i)-u(\x_j))^2
%  \label{eqn:discrete-dirichlet-energy}
%\end{equation}

\subsubsection{Limiting variational problem and PDE}

At the continuum level, we first define an analogue of the graph regularizer $R_\Gamma$. The Dirichlet energy of a function $v : \M \rightarrow \R$ is defined as
\[ R_\M(v) := \int_{\M}  |\nabla v|^2 dvol_\M,  \]
whenever $v $ is in the Sobolev space $H^1(\M)$. Also, for a smooth function $v$ we define the elliptic operator $\Delta_\M$ as 
\[ \Delta_{\M} v:= - div( \nabla v ), \]
i.e. the negative of the Laplace-Beltrami operator on $\M$. It is straightforward to show (see Section \ref{sec:bias-estimates})  that the Euler-Lagrange equation associated to the variational problem
\begin{equation}
\min_{v } \left\{ \beta R_\M(v) + \int_{\M} \int_{\R} F(v(x)- \mu(x)-s ) p(s) ds   d vol_\M (x)  \right\} \label{eqn:M-Var-Prob}
\end{equation}
is the PDE
\begin{equation}
\beta \Delta_\M v +\int_\R f(v-\mu-s) p(s) ds =0 .
\label{MPDE}
\end{equation}
Equation \eqref{MPDE} is the continuum ``homogenized'' analogue of the graph PDE \eqref{graphPDE1}. Existence and uniqueness as well as regularity properties of the solution to \eqref{MPDE} are discussed in Theorem \ref{thm:bias}. We remind the reader that $p(s)ds$ is the actual \nc distribution  function of the label noise. \nc

When passing from the graph PDE \eqref{graphPDE1} to the PDE \eqref{MPDE}, we notice that there are two terms that get homogenized. \nc On the one hand, as more feature vectors $\x_i$ are available, the graph $\Gamma$ gets denser, and the graph Laplacian $\Delta_\Gamma$ starts behaving like $\Delta_\M$. On the other hand, as more labels $y_i$ are acquired, we would like to obtain homogenization of the fidelity in \eqref{op:graphBelkin}. In the next section we present our main results relating the solutions to these two equations, i.e. \eqref{graphPDE1} and \eqref{MPDE}.

\subsection{Main results and discussion}
\label{sec:MainThms}

Our first main result establishes probabilistic error bounds for 
\[ \max_{i=1, \dots, n} | u(\x_i) - v(\x_i) |   \]
where $u$ is the solution to the graph PDE \eqref{graphPDE1} (with  the graph $\Gamma$ as defined in Section \eqref{sec:graph}), and $v$ is the solution to \eqref{MPDE}. That is, we estimate the difference between the data dependent $u$ and a homogenized function $v$. \nc
\begin{theorem}(Sample error)
\label{thm:var}
 Suppose that $u$ is the solution to the elliptic graph PDE \eqref{graphPDE1} where $\Delta_\Gamma$ is defined in \eqref{eqn:graphLap} and $v$ is the solution to the PDE \eqref{MPDE}. Assume that $\mu \in C^2(\M)$. Then for any $\delta, \zeta > 0$, with probability at least $1-4n \exp\left(  - \frac{n \delta^2 \e^{m+2}}{ C(1+ \e \delta)}\right)-4n \exp(-Cn\e^m \zeta^2) - 4n \exp(-Cn\e^m)$, 
\[ \max_{i=1, \dots, n}| u(x_i)- v(x_i) | \leq C\left( \frac{\e^2}{\beta} + \zeta + \beta \delta + \beta^{1/2} \e \right) , \]
where the constants $C$ depend only on $\mu, \eta, F,$ and $\M$.%the $C^1$ norm of $v_\beta$ and where $C(v_\beta)$ depends at most linearly on the $C^3$ norm of $v_\beta$.

\end{theorem}

One of the implications of the above result is that for $\beta$ fixed, it is possible to tune $\veps, \delta$ in terms of $n$ appropriately to deduce asymptotic uniform convergence of solutions of \eqref{graphPDE1} towards solutions of  \eqref{MPDE}.
\nc

One of the main tools used to establish Theorem \ref{thm:var} is the following maximum principle at the graph level. As the proof is straightforward, we present it immediately.

\begin{proposition}\label{prop:maxpple} (Maximum principle at the graph level) Let  $g : \R \rightarrow \R$ be a strictly increasing function and let $h \in L^2(X)$ be an arbitrary function defined on the point cloud $X$. Suppose that the function $z :X \rightarrow \R$ satisfies
	\[ - \beta \Delta_{\Gamma}  z - ( g(z + h )  - g(h)  ) \geq 0.  \] 
	Then, 
	\[ z\leq 0, \]
	i.e. the function $z$ is non-positive.
\end{proposition}

\begin{proof}
	Notice that to prove that the function $z$ is non-positive, it is enough to show that $z(\x_{{i}})\leq 0$ where ${i}$ is the index of the point $\x_i$ at which $z$ is maximized. Now, we notice that at the point $\x_{{ i}}$ we have 
	\[ \Delta_{\Gamma} z (\x_{i}) = \sum_{j=1}^n w_{{i}j} ( z(\x_{ i}) - z(\x_j) )   \geq 0.\]
	It then follows that
	\[ -( g( z (\x_{ i})   + h(\x_{ i})  )  - g(h (\x_{ i}))) \geq0   ,\]
	or equivalently,
	\[   g( z (\x_{ i})   + h(\x_{ i})  )  \leq g(h (\x_{ i}))    .\]
Since the function $g$ is strictly increasing, we conclude that $ z(\x_{ i}) \leq 0$, which concludes the proof.\nc% can not be positive for otherwise we would contradict the previous inequality.
\end{proof}
This proof is an adaptation of the proof of the maximum principle in the continuum case (see e.g. \cite{EvansBook} p. 344), with the modification that we have replaced differential operators with derivatives. We also have used the fact that we have a strictly monotone lower order term, which is often not the case in the classical setting.
\nc

Theorem \ref{thm:var} is proved by showing that the difference of the functions $u$ and $v$ (interpreting $v$ as its restriction to $X$) lies between two functions $y^-,y^+$ (referred to as barrier functions) which are uniformly close to zero, i.e., 
\begin{equation}
 y^-(\x_i) \leq  u(\x_i) - v(\x_i)   \leq y^+(\x_i)  , \quad \forall i =1, \dots, n, 
 \label{eqn:auxineq}
 \end{equation}
with
\[ \lVert y^- \rVert_\infty, \lVert y^+ \rVert_\infty \ll 1.\]
Up to some minor modifications, the functions $y^-, y^+$ take the form
\begin{align*}
\begin{split}
y^{+}_i \sim \frac{\e^2}{\beta} \left(  \int_\R f(v_i- \mu_i - s)p(s) ds - f(v_i - \mu_i - \xi_i)   \right)  + \rho
\\  y^{-}_i \sim \frac{\e^2}{\beta } \left(  \int_\R f(v_i- \mu_i - s)p(s) ds - f(v_i - \mu_i - \xi_i)   \right) -  \rho
\end{split} 
\end{align*}
where $\rho$ is a constant conveniently chosen so as to ensure that the functions $z^+:= (u-v) - y^+ $ and $z^-:= y^- - (u-v)$ satisfy the inequality required for the maximum principle at the graph level to apply with $g\equiv f$ (which then implies \eqref{eqn:auxineq}). We remind the reader that $v_i = v(x_i)$ represents pointwise evaluations of the limiting PDE \eqref{MPDE}, $\mu_i = \mu(x_i)$ are pointwise evaluations of the limiting trend function, and $\xi_i$ are the values of the noise. \nc In order to guarantee that $\rho$ can be picked to be small (so that $y^-, y^+$ are indeed uniformly small), we need an estimate for the difference between $\Delta_\Gamma v$ and $\Delta_\M v$ where $v$ is the solution to \eqref{MPDE}. Such pointwise estimate relies strongly on the regularity of the function $v$. The necessary regularity for the function $v$ in turn follows from the regularity theory of solutions to elliptic PDEs (see the last part of Theorem \ref{thm:bias} below). Apart from the pointwise estimates $\Delta_\Gamma v - \Delta_\M v$, the other probabilistic estimates that we need are used to control the expressions that appear when we apply the graph Laplacian to the functions $y^-$ and $ y^+$. After some cancellations and standard concentration inequalities, these remaining terms can be shown to be small. We emphasize that it is precisely our convenient construction of the barrier functions $y^-, y^+$, and the structure of the graph Laplacian, which ultimately allows us to handle the randomness in our problem, and bootstrap basic pointwise consistency results into convergence rates for solutions to optimization problems on random geometric graphs.

After establishing error estimates for the difference between $u$ and $v$, we obtain estimates for the difference between $v$ and a modified trend $\mu_f$ defined implicitly by
\begin{equation}
\int_\R f( \mu_f(x) - \mu(x) -s ) p(s) ds =0  , \quad \forall x\in \M.   
\label{modifTrend}
\end{equation}
The function $\mu_f$ is the solution to \eqref{eqn:M-Var-Prob} when the regularizer has been turned off (i.e. when $\beta=0$).
Notice that when the loss function $F$ has the form  $F(t)=\frac{1}{2}t^2$ the modified trend coincides with $\mu$. This is also the case for general $F$ when the noise distribution $p(s)ds$ is assumed to be symmetric. What is more, if the actual noise distribution $p$ coincides with the assumed noise distribution $e^{-F(-s)}$ then $\mu_f=\mu$. To see this notice that in that case
\[ \int_{\R} f(-s)p(s)ds= \int_{\R}F'(-s) e^{-F(-s)}ds =   \int_{\R}\frac{d}{ds}e^{-F(s)} ds  =0,  \]
which implies that if we take $\mu_f=\mu$, then $\mu_f$ indeed solves \eqref{modifTrend}.

\nc

The following result is obtained using tools from the theory of PDE.

\begin{theorem}(Regularity and approximation error)
	\label{thm:bias}
	Let $\mu_f$ be the solution to \eqref{modifTrend}. Then, there exists a unique $v \in C^2(\M)$ which solves the PDE \eqref{MPDE}. Furthermore, for $\beta$ sufficiently small, this function satisfies 
	\begin{equation} 
	  \sup_{x \in \M} | v(x) - \mu_f(x)| \leq \frac{\beta\|\Delta_\M \mu\|_\infty}{c_1}, \label{eqn:bias-est} 
	\end{equation}
	where $f'(t) > c_1$ for $t \in [-\|\mu_f\|_\infty,\|\mu_f\|_\infty]$. Furthermore, assuming that $\|\mu\|_{C^2} < \infty$ then
	\begin{equation}
	  \|v\|_{C^2} \leq C, \|v\|_{C^3} \leq C\beta^{-1/2}, \|v\|_{C^4} \leq C\beta^{-1},
	  \label{eqn:bias-uniform-est}
	\end{equation}
	where here $C$ is independent of $\beta$.
      \end{theorem}

We recall that the definitions of these norms is given in Section \ref{sec:Setup}. \nc We can combine Theorems \ref{thm:var} and \ref{thm:bias} and deduce the following error estimates between our regression function $u$ constructed by solving \ref{graphPDE1} and the modified trend $\mu_f$.

\begin{theorem}\label{thm:consistency}
Under the same assumptions in Theorem \ref{thm:var} and using the same notation there as well as that in Theorem \ref{thm:bias}, with probability greater than $1-4n \exp\left(  - \frac{n \delta^2 \e^{m+2}}{ C(1+ \e \delta)}\right)-4n \exp(-Cn\e^m \zeta^2) - 4n \exp(-Cn\e^m)$, we have
\[  \max_{i=1, \dots, n} | u(\x_i) - \mu_f(\x_i)| \leq C\left( \beta + \frac{\e^2}{\beta} + \zeta + \beta \delta + \beta^{1/2} \e\right).\]
In particular, choosing $\delta$ of order one \nc and $ \zeta = \beta = \e$ we have that with probability larger than $ 1-Cn \exp(-C n \e^{m+ 2 \nc})$ we have $\max |u(\x_i) - \mu_f(\x_i)| \leq C \e$.

%\red
%Question here: If I didn't miss anything then I think we don't actually have to take $\delta \to 0$? This is because we are multiplying the difference in the Laplacians (which, as reviewer 1 pointed out before is of order $\delta$), by the regularization parameter $\beta$. Therefore here, unlike in many of the manifold learning literature, we don't have to take $\delta$ to zero. If this is correct, then we should just take $\delta =1$ for example, and we'll get back (modulo log terms) the $n^{-1/(m+2)}$ convergence rate (which is known to be minimax).
%\nc

\end{theorem}

We note that as long as $ \left(\frac{\log(n)}{ \delta^2  n}\right)^{\frac{1}{m+2}}  \ll \e \ll \beta^{\frac{1}{2}} \ll 1$ then the previous theorem gives asymptotic consistency.  We emphasize that this theory provides clear asymptotic ranges of parameters for which consistency is achieved, and can be used to identify ranges of parameters where overfitting \emph{does not} occur in the large $n$ limit. In particular, taking $\delta$ to be of order one, and $\zeta = \beta = \e = C \left(\frac{\log(n)}{n}\right)^{1/{m+2}}$ for some large enough constant we obtain, with high probability, an overall error of estimation of
  \[\left(\frac{\log(n)}{n}\right)^{1/{m+2}} \]
  which is known to be essentially minimax optimal for the recovery of the trend function. We will discuss more on this in section \ref{sec:litera}. \nc

  One can cast the estimates in Theorems \ref{thm:var} and \ref{thm:bias} in terms of \textit{sample error} and \textit{approximation error} estimates in the following sense: the first theorem gives an estimate on the random variation (in an $L^\infty$ norm) one sees when comparing the solution of the (random) optimization problem \eqref{op:graphBelkin} and the associated mean or homogenized problem \eqref{eqn:M-Var-Prob}. The second theorem gives an estimate between the solution of the homogenized problem \eqref{eqn:M-Var-Prob} and the Bayes regressor, or in other words it provides an $L^\infty$ estimate on the bias induced (in the homogenized limit) by the regularizing term. 
\nc

We conclude this section by making a few remarks of a technical nature. 

\begin{remark}
Our results can be generalized in a straightforward way to the case where the trend function $\mu$ is smooth everywhere except on a  regular $m-1$ dimensional discontinuity set $D_\mu$ . In such case we can obtain similar error bounds for the difference between the solution to the graph PDE and the solution to the continuum PDE. Such error bounds are uniform away from the discontinuity set $D_\mu$. The reason for this is that most of our estimates are local, and even those that are not, only involve averaging at the length scale $\veps$.  It is not clear how to utilize our techniques, which are strongly grounded in the theory of elliptic PDE, to settings where the trend function is highly irregular. \nc
\end{remark}

\begin{remark}
	\label{rem:non-uniform-rho}
	[A more general statistical model\nc]
As was mentioned in Section \ref{sec:Setup}, although we state our main results assuming the data $\x_1, \dots, \x_n$ to be uniformly distributed in $\M$ and the $\xi_i$ to be identically distributed, it is completely straightforward to extend them to a more general setting. In particular, one can suppose that
\[ y_i= \mu(x_i) + \xi_i, \]
where $(\x_1,\xi_1), \dots, (\x_n, \xi_n )$ are samples from a joint distribution $\bgamma$ with a smooth marginal density for $x$ denoted by $\gamma$, conditional distributions for the noise, which vary smoothly in $x$, of the form
\[ \Prob( \xi \in ds| x) = \gamma_{x}(s)ds, \]
which are further assumed to be centered and to have bounded support. Then, Theorems \ref{thm:var},\ref{thm:bias}, and \ref{thm:consistency} continue to be true if we now let $v$ be the solution to the PDE
\[ \Delta_{\gamma} v(x) +  \int_\R   f(v(x)-\mu(x) - s) \gamma_x(s)ds =0, \quad x \in \M \]
where 
\[ \Delta_\gamma v := - \frac{1}{\gamma}div (\gamma^2 \nabla v) \]
and if we let $\mu_f$ be the function that satisfies
\[ \int_\R f( \mu_f(x) - \mu(x) - s )\gamma_x(s)ds =0  \]
for all $x \in \M$.

%\begin{remark}
%Similar results to the ones we obtain in this paper can be deduced  if we change the definition of the graph Laplacian $\Delta_\Gamma$. Take for example the random walk graph Laplacian, which is the graph Laplacian (as defined in \eqref{GraphLaplacian}) for the graph with weights
%\[ \tilde{w}_{ij}:= \frac{w_{ij}}{d_i} \]
%where $w_{ij}$ is as in \eqref{eqn:weight-def} and 
%\[ d_i := \sum_{j} w_{ij} . \]
% Proposition \ref{prop:pointwiseLap} would need to be changed for an analogous estimate (see \cite{SingerA}) and \ref{lem:averages} would not require the normalization by the $g_i$ terms.  
%
% 
%It is important to notice that our probabilistic estimates rely only on \textit{pointwise} estimates for the approximation of $\Delta_\M$ with the graph Laplacian! This contrasts with some of the related literature that we will review in section \ref{sec:literature}.
%
%\end{remark}

In this paper we have also assumed the noise in labels $y_i$ to be bounded. While our current proofs, in particular the construction of the barrier functions $y^+$ and $y^-$, \nc do not allow us directly to drop this assumption, they do serve as a basis for future improved results. In a similar way, it is likely that the assumptions we have made on the loss function $F$ can be relaxed further. 

%\red Maybe we add something about the non-linear loss function here? \nc

\end{remark}

\begin{remark}
	\label{sec:outofSamp/le}[Constructing out of sample regressors\nc] Since the regression function $u$ obtained by solving \eqref{graphPDE1} is only defined at the data points $x_1, \dots, x_n$, one needs a mechanism to extend $u$ to the whole $\R^d$ in order to use it for prediction. A simple extension mechanism proposed in Chapter 5 in \cite{Lafferty1} is defined as follows: for an arbitrary $x \in \R^d$ we consider
\[  \bm u (x) := \sum_{i=1}^n u_i \bm 1_{V_i}(x),  \]
where the $\{ V_i\}_{i \in \N}$ is the Voronoi tessellation in $\R^d$ induced by $ \{x_1, \dots, x_n \}$, that is, 
\[V_i := \{ x \in \R^d \: : \:  |x - \x_i | \leq | x - \x_j| , \quad \forall j=1, \dots, n  \}. \]
We use the above definition for points $x$ that belong to a single  Voronoi cell, and define $\bm u(x)$ to be the average of the $u(x_i)$ associated to the cells $V_i$ that $x$ belongs to. In other words, $\bm u$ is the $1$-NN extension of $u$.

 With the $L^\infty$ bounds between $u$ and $\mu_f$ that we have derived in our main theorems one can now show that $\bm u$ is uniformly close to $\mu_f$ when restricted to $\M$. Let us outline the argument. Take for simplicity a point $x \in \M$ which belongs only to $V_i$. Then, the triangle inequality implies that
\begin{align*}
\begin{split}
| \bm u(x) - \mu_f(x) | & \leq   | u_i - \mu_f(\x_i)|  + | \mu_f(\x_i)- \mu_f (x)| \
\\& \leq    \sup_{j=1, \dots, n} |u_j - \mu_f(x_j)| + Lip(\mu_f)|x-\x_i| 
\\&\leq \sup_{j=1, \dots, n} |u_j - \mu_f(x_j)| + Lip(\mu_f) diam_\M(V_i) 
\end{split}
\end{align*}
We notice that the term  $\sup_{j=1, \dots, n} |u_j - \mu_f(x_j)|$ is controlled in Theorem \ref{thm:consistency}. In turn the diameter of Voroni cells constructed from random samples can be controlled following \cite{devroye_gyorfi_lugosi_walk_2017}.  %(depending on whether we are in the semi-supervised or in the fully supervised settings).
\nc
\end{remark}

%
%\red TODO: figure out what to do with this remark. Reviewer 2 wants to move this to a conclusion section. \nc
\begin{remark}
One direction of research which is worth further exploration is to study how these ideas can be used to address similar questions to the ones explored in this paper in the context of graph models $\Gamma$ different from geometric graphs. An example of such a model is the stochastic block model where points $\x_1, \dots, \x_n$ have no geometric meaning and weights are determined randomly based on a probabilistic rule. We note \nc that large $n$ behavior of the spectra of graph Laplacians for graphs generated from a stochastic block model have been studied in \cite{RoheYu}.
\end{remark}

\section{Literature review and comparison of techniques}
\label{sec:comparison}

In this section we provide a review of relevant literature, beginning with statistical and machine learning literature related to Laplacian regularization. We then discuss recent analytical advances related to graph-based regularization. We then provide some comparison and discussion between the methods we study here and kernel and k-NN methods. Finally, we give a comparison of our analytical results with purely variational results, and demonstrate that the maximum principle method obtains more precise bounds.

\subsection{Laplacians, geometry, and regularization}
\label{sec:GraphPDEs}
%\red
%At the moment, I've just dumped a bunch of references here. We need to smooth this out and give it a narrative. We basically need to argue that people are still using and thinking about these techniques.
%\blue 
There is a significant literature related to the use of Laplacian in statistical problems. Here we complement (in a non-exhaustive way) the discussion started in the introduction.

The use of derivative penalties in parametric regression was advocated in the work of Wahba \cite{wahba1990spline}. In that context, one seeks for a polynomial spline of specified order which interpolates observed data points, and which minimizes some derivative penalty, such as the continuum Laplacian. Around 2000, a significant body of research developed the use of Laplacians to capture intrinsic geometry in statistical tasks. Often this work was carried out with the aim of designing methods which can leverage geometry to handle settings which are not fully supervised, or where there is an active component to learning. For example, \cite{Bosquet} investigated the use of density weighted regularizers in non-parametric regression problems. \cite{Belkin-Niyogi-Sindhwani} uses Laplacian eigenfunctions to conduct semi-supervised learning (so does \cite{zhou2004learning} and \cite{Lafferty1}). \cite{belkin2003laplacian} uses the Laplacian to construct dimension reduction algorithms. \cite{ng2002spectral} introduced spectral clustering, which uses the first non-trivial eigenvector of the Laplacian to construct meaningful clusters; theoretical consistency of these types of algorithms was studied in \cite{vonLuxburg}. An algorithm linking graph cuts and spectral clustering was proposed in \cite{joachims2003transductive}.  \cite{delalleau2005efficient} Studies a method which uses Laplacian regularized regression to fit labeled points, and then k-NN regression on any unlabeled or new data points. \cite{smola2003kernels} studies the use of the Laplacian to construct reproducing kernels on graphs. We remark that much of this work focuses on graph-based algorithms. We also remark that much of this work was developed simultaneously with kernel methods, and the two are often mixed in these works.

Another related body of work focuses on diffusion maps and manifold learning. These techniques describe the use of Laplacian eigenfunctions as a powerful parametric family which encapsulates underlying geometry. Important early works establishing theory for this topic include \cite{coifman2006diffusion} and \cite{Coifman7426}.

More recently, there has been renewed interest, especially in the statistics community, in utilizing derivative and difference-based methods on graphs in order to conduct regression tasks. In particular, \emph{trend filtering} \cite{tibshirani2014adaptive,wang2016trend} seeks to conduct regression tasks on a graph with an added regularization term based on function differences. This regularization term is permitted to take different forms in trend filtering, including powers of the Laplacian or TV norms. While the setup is not identical to ours, in particular due to the use of higher-order differences and $\ell^1$-type regularizers, it is very much in the spirit of the model that we study here. The use of PDE type methods to analyze these more complicated trend filtering models is the topic of future work. In a related work, \cite{gleich2015using} studies the robustness of semi-supervised methods, and proposes using sparsity type penalties on graph derivatives similar to those from trend filtering. \cite{gadde2014active} and \cite{gu2012selective} study models for active learning which build upon the graph Laplacian framework in order to identify regions where labels should be acquired.

In addition, there has recently been a broader interest in Dirichlet (and more general derivative based regularization terms) in order to conduct learning tasks. For example, \cite{osting2014minimal} studies the use of Laplacian eigenvalues on graphs in the context of Dirichlet partitions to perform generalized clustering. Bayesian inverse methods have also utilized Laplacian regularization in \cite{GarciaSanzAlonso}.

\nc

\nc

\subsubsection{Analysis of large sample limits of variational problems on graphs}\label{sec:literature}

In the past few years there has been a rapid development of a body of work borrowing ideas from the calculus of variations and PDE theory to study large sample asymptotics of optimization problems on geometric graphs closely connected to machine learning tasks. The motivation is clear: to a large extent, most of the graph-based methods for learning that are in existence can be phrased as solving either a variational problem on a graph or a graph PDE. \nc In many instances these graph PDEs involve the graph Laplacian. Said works include the study of consistency of Cheeger and ratio graph cuts on graphs \cite{TSvBLX_jmlr}, consistency of graph Laplacian spectrum \cite{GTSSpectralClustering}, and supervised and semi-supervised learning \cite{CalderSlepcev,StuartSlepcevThorpeDunlop,GarciaMurray,GarciaSanzAlonso}. In the previously listed papers, the convergence of discrete solutions to continuum counterparts is studied in the $TL^p$-metric introduced in \cite{GarciaTrillos2015} and later further studied in \cite{ThorpeTLp}. The $TL^p$ topology can be thought as $L^p$ convergence after suitable matching of the ground truth measure generating the data set $X$ and its empirical measure. The consistency of the optimization problems is studied using variational methods, and in particular the notion of $\Gamma$-convergence (a.k.a.  epi-convergence). This is a powerful notion used to establish asymptotic convergence of minimizers of optimization problems (especially in highly non-convex settings), but it does not offer direct ways to obtain rates of convergence.  

Among the papers previously listed, the paper \cite{GarciaMurray} by the authors is closely connected to this paper. There we consider an optimization problem of the form:
\[ \min_{u}  \frac{\beta}{n}\sum_{i,j}w_{ij}|u(\x_i)- u(\x_j)|  + \frac{1}{n}\sum_{i} |u(\x_i) -y_i|,  \]
which is the $L^1$ version of the problem we study here. As is well known in the image analysis community the total variation functional (the first term in the above objective function) enforces sparsity of derivatives \cite{ChambolleNovaga} and hence the above optimization problem seems more appropriate for the purposes of classification when binary labels are available (in our notation $y_i \in \{ 0,1\}$). In that paper we study the regimes of $\beta:= \beta_n$ (and how the graph connectivity $\veps$ must scale with $n$) so as to recover in the large $n$ limit the Bayes classifier with probability one. No rates of convergence are provided.

The paper \cite{SlepcevThorpe} is also related to our work.  There, $p$-Laplacian regularization for semi-supervised learning is studied. The optimization problem takes the form:

\[ \min_{u} \frac{\beta}{n}\sum_{i,j}w_{ij}|u(\x_i)- u(\x_j)|^p   \]
subject to 
\[ u(\x_i) = y_i , \quad i=1, \dots, q,  \]
where $q $ is held fixed as $n \rightarrow \infty$. The authors are able to show that when $p$ is greater than the intrinsic dimension $m$, solutions to the $p$-Laplacian regularization problem converge \textit{uniformly} to a continuum counterpart, as $n \rightarrow \infty$, which depends on the labels $y_1, \dots, y_q$ (in other words the labels are not forgotten in the limit). The uniform convergence is proved by bootstrapping the $TL^p$ convergence  obtained through variational methods by controlling the ``oscillations" of the discrete minimizers at a certain convenient length-scale. 

The paper \cite{CalderpLap} is very closely related to \cite{SlepcevThorpe} and to this paper. In particular, it obtains analogue results to \cite{SlepcevThorpe}, but using a PDE approach rather than a calculus of variations one. A maximum principle at the graph level analogous to the one that we use in this paper is a crucial tool that is later used in conjunction with general and flexible results on consistency of viscosity solutions to elliptic PDEs. In this paper we take a PDE approach as in \cite{CalderpLap}, and specifically use a maximum principle, to obtain rates for the uniform convergence of graph Laplacian regressors towards continuum counterparts. Whether similar results to the ones we present here can be obtained for regressors obtained using other regularization terms different from the graph Dirichlet energy is a question that we believe is worth exploring in the future.

Another very recent work which is related to ours is \cite{Shi-2018}. In that work the authors consider semi-supervised learning with Laplacian smoothing (with hard label constraints), but without any label noise (in particular, the labels $y_i$ are evaluations of a $C^1(\M)$ function). That paper establishes convergence rates which are analogous to ours, using similar technical tools (i.e. a maximum principle, but with boundary data). The semi-supervised aspect of their work is closely related to the harmonic extension problem (see \cite{Nadler-2009} for additional discussion on the harmonic extension problem in machine learning). We emphasize that the noisy labels that we consider here are not covered in their analysis and are not trivial to handle: see the proof of Theorem \ref{thm:var} and the discussion in Section \ref{sec:variational-approach} for further information.
\nc

\subsubsection{The linear case and connections to $k$-NN regressors and other local averaging procedures}
\label{sec:litera}

Here we draw a connection between the graph Laplacian regressor obtained by solving \eqref{graphPDE1} and the classical $k$-NN regressor. To do this, we will focus on solutions of Equation \eqref{graphPDE1} when $F(t)= \frac{1}{2}t^2$. As we will see, graph Laplacian regularization with squared error loss can be interpreted as a local averaging procedure, where the ``locality" is defined in terms of the intrinsic geometry of the graph, which in turn approximates the geometry of the underlying manifold $\M$.

Let us first briefly recall the definition of the $k$-NN regressor: For a $k \in \N$, with $k < n$, we define $N_k(\x_i)$ \nc to be the set of $k$ nearest neighbors of  $\x_i$ in the data set $X$. The $k$-NN regressor is then defined as  
\begin{equation}
 u_{k}(\x_i ):= \frac{1}{k} \sum_{\x_j \in N_k(\x_i)} y_j. 
 \label{kNNregressor}
 \end{equation}
The use of local averages in non-parametric regression goes as far back as the work \cite{tukey1948}, $k$-NN regression being a special case of this general idea. The book \cite{BookGyorfi} presents a very complete picture of many non-parametric regression techniques and dedicates a whole chapter (Chapter 6) to $k$-NN regression. Asymptotic properties of $k$-NN regressors have been a topic of investigation for several decades see \cite{BookGyorfi} and the paper \cite{DevroyeLugosi} where $L^1$ convergence towards a trend function is proved in a very general setting. More recent results like \cite{Kpotufe} prove uniform convergence towards a trend in a quite general setting where, in particular, the intrinsic dimension of the underlying ground-truth may vary; as a byproduct $k$-NN regression is shown to adapt to the local geometry of the underlying model. The paper \cite{Dasgupta} is closely related to \cite{Kpotufe}, but studies the classification problem instead. 

We now show that when $F$ is quadratic, the graph Laplacian regressor obtained by solving \eqref{graphPDE1} can be interpreted as a local averaging procedure, where now the averaging is with respect to the heat kernel on the graph; for simplicity we take  $F(t):= \frac{1}{2}t^2$. Indeed, in this case the solution to the graph PDE \eqref{graphPDE1} can be explicitly written as:
\[  u = (  \beta \Delta_\Gamma + I )^{-1} y. \]
The fact that $\Delta_\Gamma$ is self-adjoint and positive semi-definite allows us to use the spectral theorem and write:
\[ u = (  \beta \Delta_\Gamma + I )^{-1} y  = \int_{0}^\infty e^{-t( \beta \Delta_\Gamma  + I ) } y  dt. \]
Since $\Delta_\Gamma$ and $I$ commute we get:
\begin{equation}
u =  \int_{0}^\infty e^{-t} \left(  e^{-t\beta \Delta_\Gamma} y \right)  dt =  \int_{0}^\infty  \frac{e^{-t/\beta}}{\beta} \left(  e^{-t\Delta_\Gamma} y \right)  dt
\label{LapRegressorLinear}
\end{equation}
where in the final step we have made a change of variables. From this formula we can conclude a couple of things. First, we notice that the function $e^{-t \Delta_\Gamma}y $ is simply the solution to the heat equation (on the graph) with initial condition $y$ evaluated at time $t$ and can be written as
\[  e^{-t \Delta_\Gamma} y(\x_i) = \frac{1}{n} \sum_{j=1}^n K_{t}(\x_j, \x_i)y_j   ,\]
where $K_{t}(\x_j, \x_i)$ is the heat kernel on the graph at time $t$. One can then show that the function $K_{t}(\cdot, \x_i)$ is non-negative and moreover that $\frac{1}{n}\sum_{j=1}^nK_{t}(\x_j, \x_i)=1 $. From this it follows that the function $e^{-t \Delta_\Gamma} y$ is obtained by computing a local average (at length-scale $ t$) of $y$ around each point $\x_i$.  On the other hand, since the function $\frac{1}{\beta}e^{-t/\beta}$ is a probability density on $(0,\infty)$, we conclude that the graph Laplacian regressor $u$ is nothing but an average of averages of $y$ over all length-scales $t$. The weight given to each length-scale is naturally determined by the parameter $\beta$, and in particular if $\beta$ is small, more relevance is given to more local length-scales, whereas if $\beta$ is large, more relevance is given to global length scales. Notice that the structure of \eqref{LapRegressorLinear} is analogous to what we would obtain if we averaged the different $k$-NN regressors $u_k$ from \eqref{kNNregressor} over the value of $k$ to produce a regression function of the form:
\begin{equation*} \label{eqn:av-k-nn-regression} \overline{u}(x_i):= \sum_{k} g(k) u_{k}(\x_i), \end{equation*}
where in the above $g$ is some p.m.f. over $k$.

So far we have seen that the regressor $u$ that stems from graph Laplacian regularization with squared loss is obtained by averaging over local averages of the labels $y_i$ at different length-scales, and that these local averages use the intrinsic geometry of the graph (summarized in the graph heat kernel). Since in our setting the data is assumed to be sampled from a smooth manifold $\M$, it is to be expected that the graph heat kernel $e^{-t\Delta_\Gamma}$ actually behaves like the heat kernel on $\M$, $  e^{-t\Delta_\M}$. Furthermore, for small values of $t$ one has (neglecting constants) that $e^{-t \Delta_\M} \sim \frac{1}{t^{m/2}}e^{- d_\M(x,y)^2/4t}$ (see, for example, \cite{varadhan1967} or Chapter 15 in \cite{HeatKernelsManifolds}), where $ d_\M$ is the geodesic distance on the manifold $\M$. In particular, for small values of $\beta$ the regression function $u$ is expected to behave like
\[   u(x_i) \sim \int_{0}^\infty \frac{e^{-t/\beta}}{\beta}  \left( \frac{1}{n}\sum_{j=1}^n \frac{e^{-  d_\M (x_i, x_j)^2/4t}}{t^{m/2}} y_j \right)dt,   \]
which can be interpreted as a bandwidth average of local regression kernels, where the local regression kernels average over geodesic distances. The bottom line is that graph Laplacian regression (at least in the linear case) is very closely related to other local averaging regression procedures like $K$-NN regression and other methodologies that use geodesic distances to define nearest neighbors for label propagation \cite{Moscovich-2016,Lafferty1}.

We would also like to notice that \eqref{LapRegressorLinear} can be written in the form
\[  u(x_i)=  \frac{1}{n}\sum_{j=1}^{n} K_n(x_i,x_j) y_j = \left( \int_0^\infty  \frac{e^{-t/\beta}}{\beta} \left(  e^{-t\Delta_\Gamma} y \right)  dt\right)[i,j]. \nc \]	
%for  $K_n(x_i,x_j) = \int_0^\infty e^{-t/\beta}$ (obtained after integrating the graph heat kernel over $t$ weighted by $e^{-t/\beta}/\beta$). 
Written in this form, this construction of $u$ can be seen as a modified \emph{Nadaraya-Watson regression}\cite{Nadaraya-64,Watson-64}, also known as kernel regression. Indeed, to construct the Naradaya-Watson regressor one typically picks a kernel function $K$, and defines $u_{NW}(x_i) = \frac{\sum_{j=1}^{n} K(x_i-x_j) y_j}{\sum_{j=1}^{n}K(x_i-x_j)}$. Here various choices of the kernel function \nc are permitted, such as a sharp cutoff, or the Gaussian. One can also, in principle, permit the kernel to vary in $i$ (e.g. using the heat kernel after a fixed time $t$ on a manifold or graph \cite{Belkin-Niyogi-Sindhwani}), and here we see that the Laplacian regression we study in this work can be cast within that framework. \nc This can also be seen as a version of RKHS regression, without any regularization. This method is still a topic of research; for example recent works \cite{Romero-Ma-Giannakis-2016, Venkitaraman-2017} study kernel regression over graphs.
    
We also pause here to offer a comparison of performance guarantees between the regression procedure that we study and that of k-NN regression studied in \cite{Kpotufe}. In that work, under an identical regression model, it is shown that k-NN regression achieves, with high probability and ignoring logarithmic factors, the same $n^{-\frac{1}{m+2}}$ uniform convergence rate towards the trend $\mu$, with a lead constant $L$ that scales as the square of the Lipschitz constant of $\mu$. It is also shown in that work that such a rate is nearly minimax optimal, in the sense that no estimator can achieve a better rate than $L^{m/(2+m)}n^{-1/(2+m)}$. Given the connection between k-NN and the Laplacian regression methods discussed here, it is not surprising that the two have essentially the same convergence rate, as we rigorously show in our main result \ref{thm:consistency} (and the discussion below it).

%However, the proof techniques in the two cases are quite different. In particular, in this work we use the maximum principle to completely avoid the use of VC theory (or other tools from classical statistical learning theory) in bounding the variance of the regressors. \blue Indeed, in our approach the only  probabilistic estimates that we need in order to prove our results are good pointwise consistency estimates of graph Laplacians, together with some concengtration inequality fpr the sum of some specific independent random variables. This is true for the regression model with quadratic cost as well as for the  other allowed loss functions $F$. In this paper, the analytical structure that allows us to focalize our probabilistic estimates is the maximum principle, but we would like to highlight that other analytical techniques have been used in most of the references that were listed in section \ref{sec:literature} to establish a variety of estimates for solutions to different learning problems. 

It is indeed not so surprising that for small $\beta$ both $k$-NN and graph Laplacian regularization (for quadratic $F$) are so similar. In both cases the whole idea is to average ``enough" so as to remove the noise, and essentially any local averaging procedure should recover the underlying trend. However, as has been discussed throughout this paper, the Laplacian regularized supervised regression problem studied here belongs to a larger family of regularized graph-based algorithms for both supervised and \textit{unsupervised} learning. One main thesis of this work is that developing tools for the fully supervised regime will facilitate the study of other, less supervised, algorithms that utilize graph Laplacians, which are currently not amenable to analysis using standard techniques for supervised methods.%  and in our opinion it is natural and important to study further properties of graph Laplacians that are not standard in the machine learning literature, and that allow us to obtain through analytical methods, consistency results that are analogous to well established results for $k$-NN regression. \blue We have thus furthered the theoretical understanding of the use of graph Laplacians in learning while establishing links to other well known procedures. \nc

\nc

\nc

%
%For example, suppose that the noise $\xi_i \equiv 0$. Then we notice that
%
%\begin{enumerate}
%  \item Show that the Frechet derivative of the objective function is very small at $\mu_B$ (using bounds on the difference between $\Delta_\Gamma$ and $\Delta_M$ probably).
%  \item Deduce that the $u^*-\mu_B$ is also small (by uniform convexity).
%  \item Use the maximum principle to deduce an infinity bound.
%  \item Show how this is ruined by non-trivial noise: something more is needed.
%\end{enumerate}
%

\nc

\nc
%
%We use the explicit estimates that we obtain to provide an answer to the following question: if a user \red has  the possibility of choosing whether to \red acquire either one more labeled data point or several unlabeled data points, \nc which choice provides a \red greater impact on decreasing the  approximation error? Or in other words, which choice would result in a greater increase in learning?
%This is a general question \red which \nc we will return to at the end of this manuscript. In this paper we answer this question in the context of graph Laplacian regression.
%
%\nc

\nc

\subsection{PDE approach vs. \nc Variational approach}\label{sec:variational-approach}

In this section we illustrate the advantages of the PDE approach that we take in this paper and contrast it with a direct variational approach. In a variational approach one essentially uses the strong convexity \nc of the functional to be minimized in \eqref{op:graphBelkin} to derive convergence rates for minimizers. To illustrate this idea, we sketch the proof of the following proposition.

\begin{proposition} Let $F(t) = \frac{t^2}{2}$, and consider the variational problem
	\begin{equation}
	\min_u \beta R_\Gamma(u) + \frac{1}{2n} \sum_{i=1}^{n} (u(x_i) - \mu(x_i))^2 := \min_u \tilde J(u).
	\label{eqn:homog-graph-prob}
	\end{equation}
	Then the minimizer $\tilde u_n$ of \eqref{eqn:homog-graph-prob} satisfies $\|\tilde u - \mu\|_{L^2(X)} \leq C \beta^{1/2}$, where here we abuse notation and let $\mu$ represent pointwise evaluation of $\mu$ at the points in the data set $X$. Furthermore, if $u^*$ is the minimizer of \eqref{op:graphBelkin}, with the same choice of $F$, then $\|u^*-\tilde u\|_{L^2(X)} \leq C\left(\frac{\e^{1/4}}{\beta^{1/4}} + \frac{1}{(n\e^m)^{1/4}}\right)$. In turn, we have that $\|u^*-\mu\|_{L^2(X)} \leq C\left(\frac{\e^{1/4}}{\beta^{1/4}} + \beta^{1/2}\right)$. In particular, optimizing in $\beta$ we find that $\|u^*-\mu\|_{L^2(X)} \leq C\left(\e^{1/6} + \frac{1}{(n\e^m)^{1/4}}\right)$.
\end{proposition}

\begin{proof}[Sketch of proof]
	To begin, various recent results \cite{SpecRatesTrillos,BIK2} allow one to show that, with high probability, $R_\Gamma(\mu) \leq C$ as long as $\mu$ is sufficiently smooth. This, in turn, implies that $\tilde J(\tilde u) \leq C \beta$ (w.h.p.), which proves the first bound.
	
	For the second bound, recalling the definition of $J$ in \eqref{op:graphBelkin} , by using the optimality of $\tilde u$ and $u^*$, summation by parts and \ref{graphPDE1}, we compute that 
		\begin{align} \label{eqn:diff-est1}
	0 &\leq  J(\tilde u)- J( u^*) = \frac{\beta}{2n}\sum_{i,j} w_{ij}\left( |\tilde u_i-\tilde u_j|^2 - |u_i^* - u_j^*|^2 \right) + \frac{1}{2n} \sum_i (\tilde u_i - y_i)^2 - (u_i^* - y_i)^2 \\
	&= \frac{1}{2n} \sum_{i} \beta (\Delta_\Gamma \tilde u)_i \tilde u_i - (\Delta_\Gamma u^*)_i u_i^* + (\tilde u_i - y_i)^2 - (u_i^* - y_i)^2 \\
	&= \frac{1}{2n} \sum_i (\mu_i - \tilde u_i )\tilde u_i - (y_i - u_i^* )u_i^* + (\tilde u_i - y_i)^2 - (u_i^* - y_i)^2 = \frac{1}{2n}\sum_{i=1}^{n}  \left( \tilde u_i(\mu_i - y_i) + (u_i^* -\tilde u_i) y_i \right)
	%
	% \frac{1}{2n}\sum_{i=1}^{n}  \left( \tilde u_i^*(\mu_i - y_i) + (u_i^* -\tilde u_i) y_i \right), \\
	%0 &\leq \tilde J(\tilde u)-\tilde J(u^*) = \frac{1}{2n} \sum_{i=1}^{n} \left( u_i^*(y_i - \mu_i) + (\tilde u_i- u_i^*) \mu_i \right).
	\end{align}
	An identical argument, but for $\tilde J$, yields that
	\begin{equation}\label{eqn:diff-est2}
	0 \leq \tilde J(\tilde u)-\tilde J(u^*) = \frac{1}{2n} \sum_{i=1}^{n} \left( u_i^*(y_i - \mu_i) + (\tilde u_i- u_i^*) \mu_i \right).
	\end{equation}
	Adding the right hand side of \eqref{eqn:diff-est2}, which is positive, to the inequality \eqref{eqn:diff-est1} then yields
	\begin{equation}
	J(\tilde u) - J(u^*) \leq \frac{1}{2n} \sum_i (\tilde u_i - u_i^*)(\mu_i-y_i) \leq \frac{1}{2}|\langle u,\mu-y \rangle_{L^2(X)}| + \frac{1}{2}|\langle \tilde u, \mu-y \rangle_{L^2(X)}|. 
	\end{equation}
	\nc
	At this stage, one notices that the right hand side is composed of terms which are inner products between $\mu-y= \xi$ (which averages to zero at relatively small scales since $\E(\xi_i)=0$ and the $\xi_i$ are independent), and the functions $u,\tilde u$, which enjoy of some degree of regularity. In the spirit of \cite{GarciaMurray}, we then use a local averaging procedure to bound these terms. In particular, given some ball $B_\M(x,r)$, if we let $\bar u$ be the average of $u^*$ over that ball, then we may write
	\begin{displaymath}
	\frac{1}{n}\sum_{x_i \in B_\M(x,r)} u_i^*(\mu_i-y_i) = \frac{1}{n}\sum_{x_i \in B_\M(x,r)} (u_i^*-\bar u)(\mu_i-y_i) + \bar u (\mu_i-y_i).
	\end{displaymath}
	\nc
	The second of these terms we may control with high probability using concentration estimates, as long as $r$ is not smaller than $\e$, introducing some error that is in the order of $\frac{1}{\sqrt{nr^m}}$ ( as we expect to have roughly $n r^m$ terms in the sum). On the other hand, for the first term we may use the Poincar\'e inequality (which holds as long as $\e$ scales appropriately with $n$; see \cite{SpecRatesTrillos}) to deduce that \begin{displaymath}
	\frac{1}{n}\sum_{x_i \in B(x,r)} (u_i^*-\bar u)^2 \leq  \frac{C r}{n} \sum_{x_i \in B(x,r)} \sum_{x_j \in B(x,r)} w_{ij}(u_i^*-u_j^*)^2.
	\end{displaymath}
	By then using the Cauchy Schwarz inequality we obtain that, w.h.p., 
	\begin{displaymath}
	\left|\frac{1}{n}\sum_{x_i \in B(x,r)} u_i(\mu_i-y_i)\right| \leq Cr^{1/2}\left(\frac{1}{n}\sum_{x_j \in B(x,r)} w_{ij}(u_i^*-u_j^*)^2\right)^{1/2}.
	\end{displaymath}
	\nc By using a partition of unity (details of this type of argument can be found in \cite{GarciaMurray}), we may then deduce that w.h.p.
	\begin{displaymath}
	\left|\frac{1}{n}\sum_{i=1}^n u_i^*(\mu_i-y_i)\right| \leq Cr^{1/2} (R_\Gamma(u^*))^{1/2} + C\frac{1}{\sqrt{nr^m}}.
	\end{displaymath}
	An analogous bound holds for $\tilde u$. By noting that $J(\mu)$ is order one, we may deduce that $R_\Gamma(u^*) \leq C\beta^{-1}$. Similar logic applied to $\tilde J$ implies that $R_\Gamma(\tilde u) \leq C$\nc. Supposing that $r \sim \e$ (the smallest permissible scaling), we then have that $J(\tilde u) - J(u^*) \leq C \left( \frac{\e^{1/2}}{\beta^{1/2}} + \frac{1}{\sqrt{n \veps^m}} \right)$. Strong convexity \nc of $J$ in $L^2$ then implies that $\|\tilde u -u^*\|_2^2 \leq C \left( \frac{\e^{1/2}}{\beta^{1/2}} + \frac{1}{\sqrt{n \veps^m}} \right)$, which proves the desired bound.
	
	%
	%
	%
	%We can use the strict convexity of $F$ (and hence $J$) and the fundamental theorem of calculus to obtain
	%	\begin{displaymath}
	%	J(u_B) \geq J(u_B) - J(u^*) \geq c_1 \|u_B - u^*\|_{L^2}^2.
	%	\end{displaymath}
	%	If we then estimate $J(u_B)$, we find that
	%	\begin{displaymath}
	%	J(u_B) = \beta R_\Gamma(u_B) + \frac{1}{n} \sum_{i=1}^{n} F(u_B-y) = \beta R_\Gamma(u_B).
	%	\end{displaymath}
	%	Thus one only needs to estimate $\beta R_\Gamma(u_B)$. TODO: figure out, but I think this should just amount to a $\beta$ term, probably under very weak assumptions on $\e$, as described by the first reviewer.
\end{proof}

We remark that the previous proof provides an $L^2$ type estimate, and one would need to use the discrete maximum principle to upgrade to uniform convergence. Such techniques are well-known in the numerical analysis community, see e.g. \cite{Ciarlet-Raviart}. The approach presented above is elegant and straightforward, and requires rather minimal assumptions on regimes for $\e:= \e_n$. However, the rates that it proves are far worse than those that we prove here using the PDE approach. In particular, the Poincar\'e type argument, which is used to handle the noisy labels, is overly pessimistic. In addition, the proof is elegant only when $F$ is squared loss, but the details to make a similar argument work for more general loss functions are more complicated. One of the key ideas in this paper is that by leveraging the structure of the graph Laplacian (namely the maximum principle), one is able to provide much better theoretical guarantees.

%\subsection{Outline} The rest of the paper is organized as follows. In Section \ref{sec:bias-estimates} we study the approximation error from Theorem \ref{thm:bias} and establish the regularity of solutions to \eqref{MPDE}. We also present a simple heat kernel approach to obtain approximation error bounds in the linear case $ f(t)= t$ (i.e. $F(t) = \frac{1}{2}t^2$). In Section \ref{sec:var-analysis} we use our maximum principle at the graph level in conjunction with two technical lemmas (where our probabilistic estimates are presented) to prove Theorem \ref{thm:var}.  

% In section \ref{sec:semiproof} we present the proof of Theorem \ref{them:semi}.  %In section \eqref{} we present some numerical simulations whose purpose is to illustrate our theoretical findings.

\section{Proof of Main results}\label{sec:proof-of-main-results}
In this section we provide proofs of the main results. In particular, in Section \ref{sec:bias-estimates} we provide a simplified proof of Theorem \ref{thm:bias} for the case with quadratic loss, deferring the general argument to an appendix.  In Section \ref{sec:var-analysis} we use our maximum principle at the graph level in conjunction with two technical lemmas (where our probabilistic estimates are presented) to prove Theorem \ref{thm:var}. 

\nc

\subsection{PDE estimates}\label{sec:bias-estimates}

%\red At this point, the connection between the graph Dirichlet energy and the continuum energy is well-established. 
An important starting point in studying minimizers of \eqref{op:graphBelkin} is to understand properties of minimizers of \eqref{eqn:M-Var-Prob}. % For any fixed $\beta>0$, is is well known \blue (Cite Hein, Belkin, Von Luxburg etc.)  that the minimizers of the variational problem \eqref{op:graphBelkin} converge to minimizers of the continuum variational problem \eqref{eqn:M-Var-Prob}. \nc
The continuum problem \eqref{eqn:M-Var-Prob} has a rich history: namely solutions of this variational problem are known to be highly regular. One of our goals in this work is to show that the tools used to study partial differential equations, including regularity theory, are \nc quite helpful in analyzing regularized regression problems. However, establishing regularity estimates is a rather technical aspect of PDE theory. Hence, here we provide a simple proof for the case where the risk function is quadratic (and hence the necessary conditions are linear), and provide a more detailed proof for the non-linear case with some extended discussion in an appendix. \nc

\begin{proof}[Proof of Theorem \ref{thm:bias} in the case with quadratic loss]

%In this subsection we describe a different version of the previous approximation error estimates in the case where one considers $f(t)=t$ (i.e. $F = \frac{1}{2}t^2$). Here we are able to obtain such estimates using the spectral theorem and the heat kernel, which permits a more qualitative description of the procedure which is completely analogous to the one presented in Section \ref{sec:litera} in the graph setting. The theory described here is given as a further insight, and not as a separate result (as the general result in the previous section does apply here).

In this linear case, we may express the solution of the Euler-Lagrange equation \nc in the form 
\[ v = (\beta \Delta_\M + I)^{-1} \mu,\]  
which in turn can be written as:
\begin{equation}\label{eqn:heat-rep} v(x) = \int_0^\infty ( e^{-t(\beta \Delta_\M + I )}  \mu ) (x) dt = \int_{0}^\infty e^{-t} ( e^{-t\beta \Delta_\M} \mu   )(x) dt.  \end{equation}
Here we are using the spectral theorem for $\Delta_\M$. It follows that
\[ v (x)  - \mu(x) = \int_0 ^\infty e^{-t} \left( \int_\M  K_{t\beta}(y,x) (\mu(y) - \mu(x)) dy    \right) dt,  \]
where $K_{t\beta}$ is the heat kernel on $\M$ (at time $t \beta$). In particular,

\begin{align*} | v(x) - \mu(x) | &\leq  \int_0 ^\infty e^{-t} \left( \int_\M  K_{t\beta}(y,x) |\mu(y) - \mu(x)| dy    \right) dt \\
  &\leq Lip(\mu) \int_{0}^\infty e^{-t}\int_\M K_{t\beta}(y,x) d_\M(y,x)dy dt  \leq C Lip(\mu) \beta, \end{align*}
where the last inequality follows using properties of the heat kernel in $\M$, namely Gaussian upper bounds for the heat kernel on a smooth compact manifold (see, for example, \cite{varadhan1967} or Chapter 15 in \cite{HeatKernelsManifolds}). The bottom line is that
\[ \lVert v- \mu \rVert_\infty \leq C Lip(\mu) \beta,\]
where $Lip(\mu)$ is the Lipschitz constant of $\mu$ (which is finite since we have assumed that $\mu \in C^2(\M)$; see the statements of Theorem \ref{thm:bias} and Theorem \ref{thm:var}).

On a heuristic level, one can argue for the remaining bounds simply using the PDE: As $\beta \Delta_\M v = v-\mu$, and as $|v-\mu| < C\beta$, we can then deduce that $|\Delta_\M v| < C$. Similarly, we formally have that $\beta\Delta_\M^2 v = \Delta_\M v - \Delta_\M \mu$, and hence $|\Delta_\M^2 v| \leq C\beta^{-1}$. Classical theory, as discussed in the appendix, can be used to infer a $C^2$ bound from a bound on $\Delta_\M v$ and a $C^4$ bound from a bound on $\Delta_\M^2 v$. An interpolation argument gives the desired $C^3$ bound of order $\beta^{-1/2}$. Rigorously justifying these types of arguments requires that the PDE hold in a classical sense, which is often not clear a priori. More details on how one infers classical regularity of the solutions of variational problems is given in the appendix.

%We notice that formula \eqref{eqn:heat-rep} actually indicates that in this case $\|v\|_{C^k} \leq C \|\mu\|_{C^k}$, where $C$ is independent of $\beta$. We expect that similar estimates could be proved in the non-linear case, but the crude bounds on higher norms in Theorem \ref{thm:bias} were sufficient for our purposes.

\end{proof}

We note that if the heat kernel were given by a convolution, then we could pass derivatives directly to $\mu$ and from formula \eqref{eqn:heat-rep} one could immediately argue that  $\|v\|_{C^k} \leq C \|\mu\|_{C^k}$. It's likely that such bounds still hold in our case, but they were not necessary for our purposes and so we do not pursue them further.

\nc

\subsection{Probabilistic estimates}\label{sec:var-analysis}

In order to show the \textit{sample error} bounds from Theorem \ref{thm:var} we start by making some computations based on standard concentration inequalities (see, e.g., \cite{LugossiBook}). 
\begin{proposition}[Hoeffding and Bernstein inequalities]
	\label{Bernstein} 
	Suppose $U_1, \dots, U_n$ are independent real valued random variables, with mean zero, and for which $|U_i| \leq M$ for all $i=1, \dots, n$. Suppose that
	\[\frac{1}{n} \sum_{i=1}^n \Var(U_i) \leq \hat\sigma^2, \]
	for some $\hat\sigma^2$. Then, 
	\begin{itemize}
		\item (Hoeffding) For every $\delta>0$ 
		\[ \Prob \left(\left  \lvert  \frac{1}{n}\sum_{i=1}^n U_i \right \rvert > \delta  \right) \leq 2 \exp \left( \frac{-2n \delta^2}{M^2} \right). \]
		\item (Bernstein) For every $\delta>0$ 
		\[ \Prob \left(\left  \lvert  \frac{1}{n}\sum_{i=1}^n U_i \right \rvert > \delta  \right) \leq 2 \exp \left( \frac{-n \delta^2}{2 \hat \sigma^2 + 2M\delta/3 } \right). \]
	\end{itemize}

	\begin{remark}
		In most applications these inequalities are used to prove that the empirical average $\frac{1}{n}\sum_{i=1}^n U_i$ is small with high probability, so that in particular one is typically interested in choosing $\delta \ll 1$. When the estimate on the average of variances is not significantly smaller \nc than $M^2$, Bernstein's inequality does not produce any improvement over Hoeffding's.
	\end{remark}

\end{proposition}

Our first probabilistic estimates are concerned with the \textit{pointwise} convergence of $\Delta_\Gamma h$ towards $\Delta_\M h$ for a fixed regular enough function $h$. Such estimates have been obtained in the literature before (see, for example, \cite{hein2007graph,BIK2,CalderpLap}), but here we present them again emphasizing the dependence of constants on the regularity of the function $h$. In particular, we will need this explicit dependence in order to quantify the error between $u$ and $v$ in terms of the regularity estimates obtained in Theorem \ref{thm:bias} for the function $v$.

%An alternative approach for bounds on the error function may be possible. Consider the following comparison principle:
%
%\begin{proposition}
%  Suppose that $w\in L^2(\nu_n)$ is such that
%  \begin{displaymath}
%   \red -\nc \beta \Delta_\Gamma w - w \geq 0.
%  \end{displaymath}
%  Then $w \leq 0$. 
%\end{proposition}
%\begin{proof}
%  Consider the point where $w$ is maximized. At that point $\Delta_\Gamma w \red \geq \nc 0$, which then implies that at that point $-w \geq 0$, which concludes the proof. Note that we are not imposing any sort of boundary values.
%\end{proof}
%

%\blue 
%Let us consider the  PDEs: 
%
%
%\begin{equation}
%\beta \Delta u^\beta + u^\beta = \mu
%\label{PDECont}
%\end{equation}
%and
%\begin{equation}
%\beta \Delta_\Gamma u_n^\beta + u_n^\beta = y_n
%\end{equation}
%
%The \eqref{PDECont} can be interpreted classically so that in particular for all $i=1, \dots, n$ we have
%\[ \beta \Delta u^\beta(x_i) + u^\beta(x_i) = \mu(x_i) \]
%and so
%\[ \beta \Delta_\Gamma u^\beta(x_i) + u^\beta(x_i) = \mu(x_i) + \beta( \Delta_\Gamma u^\beta(x_i) - \Delta u^\beta(x_i)   )  \]
%
%
%If we consider the restriction of $u^\beta$ to the point cloud and define the function:
%
%\[ w_n := u^\beta - u_n^\beta  \]
%we see that
%\[ \beta \Delta_\Gamma w_n + w_n =  f_n  \]
%where
%
%\[ f_n(x_i):=   \mu(x_i) - y_n(x_i) + \beta(\Delta_\Gamma u^\beta(x_i) - \Delta u^\beta(x_i))  \]

\nc

\begin{proposition}(Pointwise consistency of graph Laplacian) Let $h\in C^3(\M)$.  Then, for every $\delta >0$, with probability at least  $ 1 - 2n \exp\left(  - \frac{n \delta^2 \e^{m+2}}{ C(\|h\|_{C^1},\eta,\mathcal{M})(1+ \e \delta)}\right) $, we have	
  \[  \max_{1\leq i\leq n}|\Delta_\Gamma h(\x_i) - \Delta_\M h(\x_i)|\leq  \delta +  C(m,\eta,\|h\|_{C^3})\veps, \]
  where the last constant depends at most linearly on $\|h\|_{C^3}$.
	\label{prop:pointwiseLap}
\end{proposition} 

We remind the reader that $\eta\left( \frac{|x-y|}{\e}\right)$ is the kernel describing the connectivity of our graph, $m$ is the dimension of the manifold $\M$, and the $\|\cdot\|_{C^k}$ norm measures size of the first $k$ derivatives of a function.
\nc

\begin{proof}
Associated to the function $h$ we define a function
\begin{equation}\label{eq:nl}
\L_{\veps} h(x) := \frac{2}{\tau_\eta\eps^{m+2}}\int_{\M} \eta\left( \frac{|x-\tilde x|}{\eps} \right)(h( x)-h(\tilde x)) \, dvol_\M(\tilde x), \quad x \in \M.
\end{equation}
This function can be interpreted as a non-local Laplacian of $h$. 

Fix $i\in \{1, \dots, n \}$ and denote by $U_1, \dots, U_n$ the variables
\[ U_j :=  \frac{2}{\tau_\eta \veps^{m+2}} \eta \left( \frac{|\x_i - \x_j|}{\veps} \right)( h(\x_i) -h(\x_j)). \]
Notice that given $\x_i$, we have
\[  \E(U_j | x_1 \dots x_n)\nc = \Delta_\veps h(x_i), \quad j \not = i.\]
 Also, using the bound on the support of $\eta$,
\begin{equation}\label{eqn:M-bound} |U_j - \E(U_j | x_1 \dots x_n)\nc| \leq \frac{4}{\tau_\eta \veps^{m+1}} \lVert \eta \rVert_\infty \lVert \nabla h \rVert_\infty   \end{equation}
\[Var(U_j|x_1 \dots x_n)\nc \leq \frac{4\lVert\eta \rVert_\infty \lVert \nabla h \rVert_\infty^2}{\tau_\eta^2\veps^{2m + 2}} \int_{\M} \eta\left( \frac{|\x_i-\tilde x|}{\veps}   \right) dvol_\M(\tilde x). \]
It is simple to see that for all $0<\veps<1$ and all $x \in \M$ we have
\[0< C_\M^{-1} \leq \frac{1}{\veps^m}\int_{\M}\eta \left(\frac{|x-\tilde x|}{\veps} \right) d vol_\M(\tilde x)\leq C_\M,  \]
where $C_\M$ is a positive constant. In particular it follows that
\begin{equation}\label{eqn:sig-bound}  \frac{1}{n}\sum_{j=1}^n \Var(U_j|x_1 \dots x_n)\nc \leq \frac{1}{\veps^{m+2}}C_\M \lVert \eta\rVert_\infty \lVert \nabla h \rVert_\infty^2  \end{equation}
We notice that neither $M$ nor $\sigma^2$ depend on $\x_1, \dots, \x_n$, and that the $U_j - \E(U_j | x_1 \dots x_n)\nc$ are independent random variables. We may now use Bernstein's inequality (Proposition \ref{Bernstein}), along with the definition of $\Delta_\Gamma h(\x_i)$ and Equations \eqref{eqn:M-bound} and \eqref{eqn:sig-bound}, to obtain
\[\Prob \left(   | \Delta_\Gamma h (\x_i) - \Delta_\veps h(\x_i)  |  > \delta \Bigg| x_1 \dots x_n \right)\nc \leq 2\exp\left(  - \frac{n \delta^2 \e^{m+2}}{ C(\|h\|_{C^1},\eta,\mathcal{M})(1+ \e \delta)}\right),  \] 
and by the law of iterated expectation get
\[  \Prob \left(   | \Delta_\Gamma h (\x_i) - \Delta_\veps h(\x_i)  |  > \delta \right) \leq  2\exp\left(  - \frac{n \delta^2 \e^{m+2}}{ C(\|h\|_{C^1},\eta,\mathcal{M})(1+ \e \delta)}\right).	   \]
A simple union bound implies that

\begin{equation}
\label{eq:glP}
\Prob \left( \max_{i=1, \dots, n}  | \Delta_\Gamma h (\x_i) - \Delta_\veps h(\x_i)  |  > \delta \right) \leq 2n \exp\left(  - \frac{n \delta^2 \e^{m+2}}{ C(\|h\|_{C^1},\eta,\mathcal{M})(1+ \e \delta)}\right).
\end{equation}

Now we claim that for all $h\in C^3(\M)$ and all $x\in \M$
\begin{equation}\label{eq:nlcon}
|\L_{\eps} h(x) - \Delta_\M h(x)|\leq  C_m Lip(\eta) \|h\|_{C^3}\eps.
\end{equation}
We first replace $\L_\veps h$ with a version of it that uses the geodesic distance on $\M$ rather than the Euclidean distance. More precisely, we set
\[ \widetilde{\Delta}_\veps h (x): = \frac{2}{\tau_\eta \eps^{m+2}}\int_{\M} \eta\left( \frac{d_\M(x,\tilde x)}{\eps} \right)(h(x)-h(\tilde x))\,  d vol_\M(\tilde x),
\]
where $d_\M(x,\tilde{x})$ is the geodesic distance between two points $x, \tilde x$ in $\M$. Now, as long as $|x-\tilde x| \leq c $ for some small enough $c$ (that only depends on $\M$) we have that
\[ | d_\M(x, \tilde{x}) - |x-\tilde x|  |  \leq C(\M) |x-\tilde{x}|^3,   \] 
where $C(\M)$ is a constant that depends on $\M$; see for example Proposition 2 in \cite{SpecRatesTrillos}.

From this and the Lipschitz continuity of $\eta$ we can conclude that if $|x-\tilde{x}| \leq 2\veps$ then,
\[ \left \lvert   \eta \left( \frac{d_\M (x, \tilde x)}{\veps} \right)  - \eta \left( \frac{|x-\tilde x|}{\veps} \right)   \right \rvert  \leq    C(\M)Lip(\eta)   \veps^2      . \]
On the other hand, if $2 \veps <|x-\tilde{x}| $, we must also  have $d_\M(x, \tilde x)>\veps$. Therefore, in all cases we have
\[ \left \lvert   \eta \left( \frac{d_\M (x, \tilde x)}{\veps} \right)  - \eta \left( \frac{|x-\tilde x|}{\veps} \right)   \right \rvert  \leq  C(\M)Lip(\eta) \veps^2 \mathbf{1}_{B_\M(x, 2 \veps)} (\tilde x),     \]
from where it follows that for all $x \in \M$,
\begin{align*}
\begin{split}
| \Delta_{\veps} h(x) - \widetilde \Delta_{\veps} h (x) | & \leq \frac{2}{\tau_\eta \veps^{m+2}} \int_{\M \cap B_\M(x,2\veps)}\left| \eta \left( \frac{|x-\tilde x|}{\veps} \right) - \eta \left( \frac{d_\M(x, \tilde x)}{\veps} \right)  \right| \lvert h(x)- h(\tilde x) \rvert d vol_\M(\tilde x )   
\\ & \leq  C(\M)  \frac{Lip(\eta)}{\tau_\eta} \lVert \nabla h \rVert _\infty \veps.
\end{split}
\end{align*}

Let us now compare $\widetilde \Delta_\veps h (x)$ with $\Delta_\M h(x)$. For that purpose we use the exponential map at the point $x$,
\[exp_{x}  : B_m(0, \veps) \rightarrow B_\M(x \nc, \veps)\] 
which takes tangent vectors $v$ at $x$ with norm less than $\veps$ into points $\exp_x (v)$ in $\M$ that are within geodesic distance $\veps$ of $x$. Let $H$ be the composition $H:= h \circ \exp_x(v)$, i.e. the function $h$ written in normal coordinates around $x$. The regularity of $h$ and $\M$ implies that $H$ is also regular, and using a Taylor expansion around the origin we get 
\[  H(v) = H(0) + \langle \nabla H (0) , v \rangle + \frac{1}{2}\langle  D^2H(0) v , v \rangle  + r(v),   \]
where the remainder $r$ is a function that satisfies:
\[  | r(v)| \leq C \veps^3 , \quad \forall v \in B_m(0, \veps).    \]
The constant $C$ depends on $\M$ and the third derivatives of $h$. We emphasize that, by the integral form of the Taylor remainder theorem, the constant in this expression scales at most linearly in the third derivatives of $h$. \nc In normal coordinates we can then write
\begin{align*}
\widetilde \L_{\eps} h(x) &= \frac{2}{\tau_\eta \eps^{m+2}}\int_{B_m(0,\eps)} \eta\left( \frac{|v|}{\eps} \right)(H(v)-H(0))J_x(v)\, dv\\
&=\frac{2}{\tau_\eta \eps^{m+2}}\int_{B_m(0,\veps)} \eta\left(\frac{|v|}{\veps} \right) \langle \nabla H (0) , v \rangle J_x(v)\, dv + \frac{2}{\tau_\eta \eps^{m+2}}\int_{B_m(0,\veps)} \eta\left(\frac{|v|}{\veps} \right) \langle D^2H (0) v , v \rangle J_x(v)\, dv
\\& + \frac{2}{\tau_\eta \eps^{m+2}}\int_{B_m(0,\veps)} \eta\left(\frac{|v|}{\veps} \right)r(v) J_x(v)\, dv.
\end{align*}
We know that the Jacobian of the exponential map $J_x$ satisfies
\[ J_x(v) =  1 + O(|v|^2), \]
(see Section 2.2. in \cite{BIK2}) so we can actually write
\begin{align*}
\widetilde{\Delta}_\veps h (x) & = \frac{2}{\tau_\eta \veps^{m+2}} \int_{B_m(0, \veps)}\eta \left( \frac{|v|}{\veps} \right)\langle \nabla H(0), v \rangle dv +          \frac{1}{\tau_\eta \veps^{m+2}} \int_{B_m(0, \veps)}  \eta \left( \frac{|v|}{\veps} \right) \langle D^2H(0) v , v \rangle dv  + R,  
\end{align*}
where $|R| \leq C(\|h\|_{C^3},m,\eta) \e$, with the constant depending at most linearly in $\|h\|_{C^3}$. \nc We notice that the first term on the right hand side of the above expression drops due to the radial symmetry of the kernel, and also that the second term is equal to
\[ trace( D^2H(0)) = \Delta H(0) =  \Delta_\M h(x)   .\] 
The bottom line is that, as anticipated in \eqref{eq:nlcon},
\[ | \Delta_\veps h (x) - \Delta_\M h(x) |  \leq   |   \Delta_\veps h (x)  - \widetilde \Delta_\veps h (x)  |   +  | \widetilde{\Delta}_\veps h (x) -\Delta_\M h (x)  | \leq   C(\|h\|_{C^3},m,\eta) \veps. \]

%Using the Taylor expansions $J_x(\eps v) = 1+ O(\eps^2)$, $p(\eps v) = p(0) +  \nabla p(0)\cdot v  \eps + O(\eps^2)$, and 
%\[w(\eps v) - w(0) = \nabla w(0)\cdot v \eps + \frac{1}{2}v \cdot \nabla^2 w(0) v \eps^2 + O(\|w\|_{C^3(B(0,\eps))}\eps^3),\]
%a standard computation yields
%\begin{align*}
%\L^{i,\eps}_{nl} u(x) = \sigma_\eta \left( \nabla w(0) \cdot \nabla p(0) + \frac{p(0)}{2}\Delta w(0) \right) + O(c_3\eps)= \frac{\sigma_\eta}{2p} \div ( p^2 \nabla w )\big\vert_{v=0}+ O(c_3\eps).
%\end{align*}
%where
%\[c_3 = 1+\|w\|_{C^3(B(0,\eps))}.\]
%The proof is completed by recalling \eqref{eq:error} and noting that $\frac{\sigma_\eta}{2p} \div ( p^2 \nabla w )\big\vert_{v=0} = \Delta_\rho u(x)$.

Combining \eqref{eq:glP} and \eqref{eq:nlcon} we deduce that with probability greater than $ 1 - 2n \exp\left(  - \frac{n \delta^2 \e^{m+2}}{ C(\|h\|_{C^1},\eta,\mathcal{M})(1+ \e \delta)}\right) $, 
\[  \max_{i=1, \dots, n} |\Delta_\Gamma h(\x_i) -  \Delta_\M h(\x_i)|\leq     \delta + C(\|h\|_{C^3},m,\eta)\veps.  \]

\end{proof}

Our next result will allow us to show that the functions $y^-$ and $y^+$ mentioned in \eqref{eqn:auxineq} and defined explicitly in \eqref{y+-}, are uniformly small.

\nc

\begin{lemma}
	\label{lem:averages}
	Let $ h: \M \rightarrow \R $ be a smooth function. For each $i=1, \dots, n$ let $E_i$ be defined as
	\[ E_i:= \sum_{j=1}^n \frac{\eta_{ij}}{g_j} \int_\R ( f( h(\x_j) - s)   - f(h(\x_j)- \xi_j)  )p(s) ds,  \]
	where
	\[ \eta_{ij}:= \frac{1}{n \veps^m } \eta \left( \frac{|\x_i-\x_j|}{\veps} \right) , \quad \text { and } \quad g_i := \sum_{l=1}^n \eta_{il}. \]
	Let $\zeta >0$. Then with probability greater than $1-2n \exp\left( -cn\veps^m \zeta^2 \right)    - 2n \exp(-cn \veps^m)$, 
	\[    \left | E_i \right |  \leq  \zeta , \quad \forall i=1, \dots, n.\]
\end{lemma}

\begin{proof}
Fix $i=1, \dots, n$ and let $U_j$ be the random variables
	\[ U_j :=   \frac{n \eta_{ij}}{g_j}f( h(x_j) -\xi_j), \quad j=1, \dots, n.\]
Conditioned on $x_1, \dots, x_n$, the variables $U_1, \dots, U_n$ are independent and satisfy
	\[ |U_j| \leq \frac{1}{\veps^m}\lVert\eta\rVert_\infty \frac{M_{h,f}}{G_{\vec{x}}}, \]
	where 
	\[ M_{h,f} :=  \sup_{x \in \M} | f( h(x) \pm \sigma  ) | , \]
and where 
\[ G_{\vec{x}}:= \min_{j=1, \dots,n} g_j.  \]
Moreover, 
\[ \E\left( \frac{1}{n}\sum_{j=1}^n U_j \Bigg| x_1 \dots x_n  \right) \nc= \frac{\eta_{ij}}{g_j}\int_\R f( h(x_j)-s )p(s)ds  ,   \]
	and
	\[ \frac{1}{n} \sum_{j=1}^n \Var(U_j | x_1 \dots x_n)\nc \leq \frac{M_{h,f}^2}{n \veps^{2m}  G_{\vec{x}}^2}\sum_{j=1}^n \eta^2\left(\frac{|\x_i-\x_j|}{\veps} \right) =  \frac{M_{h,f}^2 \lVert\eta \rVert_\infty}{\veps^m G_{\vec x}^2}g_i  \leq \frac{M_{h,f}^2 \lVert \eta\rVert_\infty}{\veps^m  G_{\vec x}^2} \widetilde{G}_{\vec x},  \]
 where
\[ \widetilde{G}_{\vec x}:= \max_{i=1, \dots, n} g_i. \]
Bernstein's inequality (Proposition \ref{Bernstein}) then implies that:
	\begin{align*} \Prob(|E_i| \geq \zeta | x_1 \dots x_n) &= \Prob \left( \left| \frac{1}{n}\sum_{j=1}^n U_j - \frac{1}{n}\sum_{j=1}^N  \E(U_j | x_1 \dots x_n)\right|  \geq \zeta \Bigg| x_1 \dots x_n \right ) \nc \\
	  &\leq 2\exp\left(  \frac{-n \zeta^2}{\frac{2 \lVert \eta \rVert_\infty M_{h,f}^2}{\veps^m G^2_{\vec x}} \widetilde G_{\vec x}  + \frac{2\lVert \eta \rVert_\infty M_{h,f} }{3 \veps^m G_{\vec x}}  \zeta   } \right).\end{align*}
Using a simple union bound we deduce that
\[ \Prob \left( \max_{i=1, \dots, n}|E_i | \geq \zeta \Bigg| x_1 \dots x_n \right) \nc\leq 2 n\exp\left(  \frac{-n \veps^m \zeta^2}{\frac{2 \lVert \eta \rVert_\infty M_{h,f}^2}{ G^2_{\vec x}} \widetilde G_{\vec x}  + \frac{2\lVert \eta \rVert_\infty M_{h,f} }{3 G_{\vec x}}  \zeta   } \right),  \]
and by the law of iterated expectation we obtain
\[ \Prob\left( \max_{i=1, \dots, n}|E_i | \geq \zeta  \right) \leq 2n \E \left(\exp\left(\frac{-n \veps^m \zeta^2}{\frac{2 \lVert \eta \rVert_\infty M_{h,f}^2}{ G^2_{\vec x}} \widetilde G_{\vec x}  + \frac{2\lVert \eta \rVert_\infty M_{h,f} }{3 G_{\vec x}}  \zeta   } \right)  \right).  \]

Now, the only terms in the above expression that depend on $\x_1, \dots, \x_n$ are $G_{\vec x}$ and $G_{\tilde x}$. These however can be showed to be bounded below and above by positive constants with very high probability. Indeed, we first notice that for all $\veps<1$ and all $x \in \M$ we have
\[ 0 < C_\M^{-1} \leq K_\veps (x):= \frac{1}{\veps^m}\int_{\M} \eta \left(  \frac{|x- \tilde x|}{\veps}\right) d vol_\M (\tilde {x})   \leq C_\M,\]
for some positive constant $C_\M$.  On the other hand, using Hoeffding's inequality we get that
\[
\Prob\left( \max_{i=1, \dots, n}| g_i - K_{\veps}(\x_i)| \geq \frac{1}{2C_\M} \right) \leq 2 n \exp\left( - c n\veps^m  \right),  
\]
from where it follows that except on a set with probability less than $2n \exp(-c n \veps^m)$,  we have
\begin{equation}
 \frac{1}{2C_\M}  \leq G_{\vec x} \leq \widetilde{G}_{\vec x} \leq 2 C_\M.
\label{LB:G}
\end{equation}
 Therefore, 
	\[ \Prob\left( \max_{i=1, \dots, n}|E_i | \geq \zeta  \right) \leq  2n \exp\left( -cn\veps^m \zeta^2 \right)    + 2n \exp(-cn \veps^m).  \]

\end{proof}

We are ready to prove Theorem \ref{thm:var}.

\begin{proof}[Proof of Theorem \ref{thm:var}]
	
Let us first introduce some notation. For a function $g: X \rightarrow \R \nc $ we denote by $g_i$ the value of the function at $\x_i$, i.e. $g(\x_i)$. 	Also, we will restrict the solution of \eqref{MPDE} and its Laplacian $\Delta_\M v$ to the point cloud $X$, so in particular we will treat $v$ and $\Delta_\M v$ as functions defined on $X$.

First, we notice that, by \eqref{MPDE} \nc
\[ \beta \Delta_\Gamma v + \int_\R  f(v- \mu - s) p(s) ds = \beta(\Delta_\Gamma v -\Delta_\M v) ,   \]
at all points in $X$. Let us denote by $a: X \rightarrow \R$ the right hand side of the above expression (i.e. $\beta$ times the difference between $\Delta_\Gamma v$ and $\Delta _\M v$). By Proposition \ref{prop:pointwiseLap} and Theorem \ref{thm:bias} we know that with probability at least $ 1 - 2n \exp\left(  - \frac{n \delta^2 \e^{m+2}}{ C(\mu,\eta,\mathcal{M})(1+ \e \delta)}\right) =: 1-p_{n,\delta}$ we have that
\begin{equation}\label{eqn:a-i-bound} \max_{i=1, \dots, n} | a_i| \leq \beta( \delta + C\beta^{-1/2}\e). \end{equation}

Now, let $w:= u- v$. Then, by \eqref{graphPDE1} and \eqref{MPDE} \nc,
\begin{align}
\begin{split}
- \beta \Delta_\Gamma w & =  - \beta \Delta_\Gamma u  + \beta \Delta_\Gamma v  
\\& = f(u - y) - \int_\R f(v- \mu - s ) p(s) ds + a
\\&=  f(u- \mu - \xi )- \int_\R f(v- \mu - s ) p(s) ds +a
\end{split}
\end{align}

Let us define the functions $y^{+}$ and $y^{-}$ on $X$ respectively by
\begin{align}
\begin{split}
y^{+}_i:= \frac{\e^2}{\beta g_i} \left(  \int_\R f(v_i- \mu_i - s)p(s) ds - f(v_i - \mu_i - \xi_i)   \right)  + \rho
\\  y^{-}_i:= \frac{\e^2}{\beta g_i} \left(  \int_\R f(v_i- \mu_i - s)p(s) ds - f(v_i - \mu_i - \xi_i)   \right) -  \rho
\end{split}
\label{y+-} 
\end{align}
where $g$ is as defined in Lemma \ref{lem:averages} and $\rho$ is a constant that will be chosen later on. Indeed, we will show that with the appropriate choice of $\rho$, the following holds at all point in $X$:
\[ y^- \leq w \leq y^+.\]

We focus on showing $w\leq y^+$, the other inequality obtained in a completely analogous way. To see that $w \leq y^+$, we will actually show that for an appropriate (small) value of $\rho$, the function 
\[  z:= w - y^{+} \]
satisfies the inequality
\begin{equation}
 - \Delta_\Gamma z  - ( f( z+ v- \mu  - \xi)  - f(v-\mu - \xi)   ) \geq 0,
 \label{IneqAuxProof}
\end{equation}
from where it follows, thanks to the maximum principle (Proposition \ref{prop:maxpple}), that $z \leq 0$. Let us then focus on showing \eqref{IneqAuxProof}. First, a direct computation shows that
\begin{align}
\begin{split}
 (\beta\Delta_\Gamma  y^+ )_i  &=  \int_\R f(v_i- \mu_i - s)p(s) ds - f(v_i - \mu_i - \xi_i)  
 \\&- \sum_{j=1}^N \frac{\eta_{ij}}{g_j}  \left( \int_\R f(v_j -\mu_j -s) p(s) ds - f(v_j - \mu_j - \xi_j)  \right).
 \end{split}       
\end{align}
where in the above we are using $\eta_{ij}$ as defined in Lemma \ref{lem:averages}. It follows that
\begin{align}
\begin{split}
(-\beta \Delta_\Gamma z)_i &=  f( u_i - \mu_i  -\xi_i ) -    f(v_i - \mu_i - \xi_i)   - \sum_{j=1}^n\frac{\eta_{ij}}{g_j}\int_\R ( f(v_j - \mu_j - s)   - f(v_j - \mu_j - \xi_j)  )p(s) ds + a_i
\\& =  f( w_i + v_i- \mu_i  -\xi_i ) -    f(v_i - \mu_i - \xi_i)   - \sum_{j=1}^n\frac{\eta_{ij}}{g_j}\int_\R ( f(v_j - \mu_j - s)   - f(v_j - \mu_j - \xi_j)  )p(s) ds +a_i,
\end{split}
\end{align}

Since $v- \mu$ is a bounded function (in particular thanks to their regularity as it follows from Theorem \ref{thm:bias} and by the assumptions on $\mu$), Lemma \ref{lem:averages} implies that, with probability at least $1-2n \exp(-cn\e^m \zeta^2) - 2n \exp(-cn\e^m) =: 1-p_{n,\zeta}$ we have
\[    \left | \sum_{j=1}^n \frac{\eta_{ij}}{g_j}\int_\R ( f(v_j - \mu_j - s)   - f(v_j - \mu_j - \xi_j)  )p(s) ds \right |  \leq  \zeta , \quad \forall i=1, \dots, n.\]
Hence, by using \eqref{eqn:a-i-bound}, with probability at least $1-p_{n,\delta} - p_{n,\zeta}$, for all $i$ we have:
\begin{align*}
  (-\beta \Delta_\Gamma z)_i \geq   f( w_i + v_i- \mu_i  -\xi_i ) -    f(v_i - \mu_i - \xi_i)  -  \zeta -\beta(\delta + C\beta^{-1/2}\varepsilon),
\end{align*}
which can be rewritten as
\begin{equation}\label{eqn:est-1}
  -\beta \Delta_\Gamma z   - ( f(z+ v- \mu - \xi )  - f(v- \mu -\xi)  )  \geq f(w + v -\mu - \xi ) - f(z + v -\mu - \xi )  - \zeta - \beta(\delta + C\beta^{-1/2}\varepsilon).
\end{equation}

Now, notice that $\rho$ can be chosen in such a way that 
\[ y^+ \geq 0. \]
Indeed, since we have assumed that the noise $\xi$ is bounded, we can conclude that  $y^+ \geq -C_2 \frac{\e^2}{\beta} + \rho$ for some constant $C_2$, from where it follows that if $\rho$ is chosen to be larger than $C_2 \frac{\e^2}{\beta}$ we can conclude that $y^+ \geq 0$. In particular, for such choice of $\rho$ we have $w= z+ y^+\geq z $ and thus by the fundamental theorem of Calculus:
\[ f(w + v -\mu - \xi ) - f(z + v -\mu - \xi ) = \int_{s_1}^{s_2} f'(s)  ds \geq c (s_2 - s_1) = c y^+, \]
for some constant $c>0$ (using the assumed strict monotonicity of $f$) and where 
\[ s_2:=  w + v -\mu -\xi , \quad s_1:= z + v -\mu - \xi  . \]
Plugging this back into \eqref{eqn:est-1} we deduce that (with probability at least $1-p_{n,\delta}-p_{n,\zeta}$)
\[  - \beta \Delta_\Gamma z - ( f(z+ v- \mu - \xi )  - f(v- \mu -\xi)  ) \geq cy^+  - \zeta-\beta\delta-C\beta^{1/2}\e \geq c \rho-  cC_2 \frac{\e^2}{\beta}- \zeta-\beta\delta - C\beta^{1/2}\varepsilon).\]
Hence if we let $\rho$ be defined according to 
\[   \rho := \frac{C_2\e^2}{\beta} + \frac{\zeta + \beta \delta +C\beta^{1/2} \e}{c},  \]
we conclude that, with probability at least $1-p_{n,\delta}-p_{n,\zeta}$
\[  -\Delta_\Gamma z - (  f(z+ v -\mu - \xi)   - f(v-\mu -\xi)   ) \geq 0, \]
as we wanted to show. Repeating this argument for $y^-$, and using a union bound completes the proof.

\end{proof}

\appendix

\section{Proof of theorem \ref{thm:bias}}

Here we give an outline of the proof of  Theorem \ref{thm:bias} in the case with general loss function. We attempt to offer some additional discussion and pointers to important inequalities that are used in this process, but do not attempt to provide all the details (see e.g. the reference \cite{Gilbarg-Trudinger} for complete details). \nc

  \begin{proof}

    Given the continuum variational problem \eqref{eqn:M-Var-Prob}, the first question is whether a unique minimizer exists. A typical modern approach is to consider minimizing this functional over a wide class of functions (e.g. $H^1$, the broadest class of functions for which the Dirichlet energy is finite). Since the functional is convex and coercive, one can generally infer the existence of \nc a solution using ``soft'' (i.e. non-constructive) methods. This is done by using, e.g., weak compactness of bounded sets in $H^1$ along with weak lower semi-continuity of the functional. Alternatively, in the cases we're considering one can use other non-constructive methods, such as Lax-Milgram or Browder-Minty (see Chapter 6, Theorem 3 in \cite{EvansBook} and \cite{Renardy-Rogers} Section 10.3), to infer the existence of minimizers. Uniqueness usually follows directly from strong convexity \nc of the functional. Directly using these methods, we may infer the existence and uniqueness of an $H^1$ function minimizing \eqref{eqn:M-Var-Prob}.

  Once one has assured the existence and uniqueness of a ($H^1$) solution to the problem, we would like to study finer properties of the solutions. To do this, we first notice that by taking variations in \eqref{eqn:M-Var-Prob}, that is by letting $v_\e = v + \e w$, and then considering $\lim_{\e \to 0} \frac{J(v_\e) - J(v)}{\e}$, \nc we have for any $w \in H^1$
  \begin{displaymath}
     \int_{\M} \beta \nabla w \cdot \nabla v + \left(\int_{\R}f(v-\mu -s ) p(s) ds\right) w dvol_\M(x)= 0,
  \end{displaymath}
  where $v$ is the minimizer of \eqref{eqn:M-Var-Prob}. Note that if $v$ were sufficiently regular (e.g. $C^2$) then we could use integration by parts in the first term and the fundamental theorem of the calculus of variations to infer \eqref{MPDE}. At the moment, given only that $v \in H^1$, we simply can say that $v$ is a \emph{weak solution} of the PDE \eqref{MPDE}.

  Several avenues are available at this stage to demonstrate that the optimizer $v$ is more regular. First, we notice, in our case, that truncating $v$ at any value above $\max \mu + \sigma$ and below $\min \mu- \sigma$ (we recall that $p$ is supported in $[-\sigma, \sigma]$) will decrease the objective value in \eqref{eqn:M-Var-Prob}. This implies that
  \begin{equation}\label{eqn:crude-inf-bias-bd}
    \|v\|_{L^\infty(\M)} \leq  \|\mu\|_{L^\infty(\M)} + \sigma.
  \end{equation}
  
  Next, various tools are available for establishing regularity of elliptic equations. For example, Theorem 2 in Chapter 6 in \cite{EvansBook} states that any weak solution of $\Delta_\M w + g = 0$ (for an arbitrary $g$) will satisfy 
  \begin{equation}\label{eqn:Evans-lemma}
   \|w\|_{H^{r+2}(\M)} \leq C(\|g\|_{H^r(\M)} + \|w\|_{L^2(\M)}),
 \end{equation}
 where the inequality is only meaningful when $g$ belongs to the Sobolev space $H^r(\M)$ (i.e. the largest space of functions where one can make sense of $r$-th order ``weak" derivatives which are squared integrable). This is proved by using the weak elliptic equation to provide a priori bounds on difference quotients of the function $w$. Using a version of the chain rule in higher-order Sobolev spaces (namely that $\|f \circ v\|_{H^r} \leq \|f\|_{C^r} \|v\|_{H^r}$, see e.g. \cite{Bourdaud} or \cite{Isaia}), and using the fact that $\mu$ and $f$ are smooth, we can take $g=\int_{\R} f( v(\cdot) - \mu(\cdot ) -s) p(s) ds $ and rewrite this estimate in the following way:
 \begin{displaymath}
   \|v\|_{H^{r+2}(\M)} \leq C(\|v\|_{H^r(\M)} + \|v\|_{L^\infty(\M)}),
 \end{displaymath}
 where here we remark that the constants in the previous line will depend on $\beta,r,\M,f,$ and $\mu$. Iterating this inequality then gives that $v \in H^r$ for any positive integer $r$.

 Once one has established Sobolev regularity, we may use Morrey's inequality (see, e.g., Theorem 6 in Chapter 5 of \cite{EvansBook}), which allows one to infer that for any $k \in \N$ there exists an $r$ so that $\|v\|_{C^k} \leq C \|v\|_{H^r}$. This then implies that the minimizer of the problem \eqref{eqn:M-Var-Prob} is in fact infinitely differentiable, and is a classical solution of \eqref{MPDE}.

 Once one has a more regular solution, a variety of techniques are available to demonstrate a priori-bounds on different derivatives (such as the bounds \eqref{eqn:bias-est} and \eqref{eqn:bias-uniform-est} ). For example, the classical maximum principle (see e.g. Theorem 2 in Chapter 6 of \cite{EvansBook}) states that, given a $C^2$ function $w$, if  $-\Delta_\M w \geq 0$  on a set $E\subset \M$ then $w$ attains its maximum on the boundary of $E$. We may use this to prove the estimate \eqref{eqn:bias-est} as follows:  we note that for any point where $v \geq \mu_f$ we have
\begin{align*}
 \beta \Delta_\M  (v - \mu_f) & = - \int_{\R} f(v - \mu -s) p(s) ds - \beta \Delta_\M \mu_f 
 \\ & =\int_{\R} \left(  f(\mu_f - \mu -s) -  f(v - \mu -s) \right) p(s)ds  - \beta \Delta_\M \mu_f 
 \\& \leq - c_1(v- \mu_f) - \beta \Delta_{\M} \mu_f
\end{align*} 
where in the second equality we have used the definition of $\mu_f$ in \eqref{modifTrend}, and in the inequality we have used \eqref{eqn:crude-inf-bias-bd} and the fact that on a bounded interval we have $f' \geq c_1 >0$ for some constant $c_1$ (which follows from the strict monotonicity of $f$). Now, suppose for the sake of contradiction that the set
\[ E = \{ x \in \M : v-\mu_f > \frac{\beta \|\Delta_\M \mu_f\|_\infty}{c_1}\} \]
is non-empty. Notice that this is an open set given that both $v$ and $\mu_f$ are continuous. From the above computations it follows that on $E $ we have that
	\begin{displaymath}
	\beta \Delta_\M (v-\mu_f) \leq 0.
	\end{displaymath}
	This implies (by the classical maximum principle) that $v-\mu_f$ when restricted to $\overline{E}$ attains its maximum on the boundary of $E$. However, since $\M$ is a manifold without boundary, $\partial E$ takes the form $\partial E = \{x : v-\mu_f = \frac{\beta\|\Delta_\M \mu_f\|_\infty}{c_1}\}$, and we conclude that the maximum value that $v-\mu_f$ can take in $\overline{E}$ is $\frac{\beta}{c_1} \|  \Delta_\M \mu_f \|_\infty$. However, this contradicts the fact that $E$ was nonempty (where in theory $v-\mu_f$ achieves values higher than $\frac{\beta}{c_1} \|  \Delta_\M \mu_f \|_\infty$ ). This provides the desired upper bound. The lower bound is deduced analogously. This proves \eqref{eqn:bias-est} and by directly using the Euler-Lagrange equation \eqref{MPDE}, we then have that $\|\Delta_\M v\|_\infty \leq C$.
	
	Now, to prove further bounds, we need a priori estimates in stronger norms. Many types of estimates are available, but we focus on two: H\"older type estimates (due to De Giorgi, Nash and Moser), and Schauder estimates. The classical H\"older estimates state that any $H^1$ solution of $\Delta_\M w +g =0$ will satisfy (see Theorem 8.24 in \cite{Gilbarg-Trudinger})
	\begin{displaymath}
	  \|w\|_{C^{0,\alpha}(\M) } \leq C(\|w\|_{L^2(\M)} + \|g\|_{L^\infty(\M)}),
	\end{displaymath}
	for some appropriately chosen $\alpha > 0$ (here $C^{0,\alpha}$ denotes the space of $\alpha$-H\"older continuous functions). We can use this to infer that $\|v\|_{C^{0,\alpha}(\M)} < C$, with $C$ independent of $\beta$. On the other hand, the classical Schauder estimates (see Theorem 6.6 in \cite{Gilbarg-Trudinger}) state that for a $C^{2,\alpha}$ solution of $\Delta_\M w + g = 0$ we have the bound
\begin{displaymath}
\|w\|_{C^{2,\alpha}(\M)} \leq C(\|w\|_{C^0(\M)} + \|g\|_{C^{0, \alpha}(\M)}),
\end{displaymath}
where here $C^{2,\alpha}$ is the space of functions with $\alpha$-H\"older continuous second derivatives\nc.
By applying this to $w=v$, and noting that by the smoothness of $f$ we have that $\|f \circ v\|_{C^\alpha} \leq \|\nabla f\|_\infty\|v\|_{C^\alpha}$, we may then apply these estimates to infer that $\|v\|_{C^{2,\alpha}(\M)} \leq C$, independent of $\beta$. In turn, by considering $w = \Delta_\M v$, we may again apply the Schauder estimate to conclude that $\|v\|_{C^4(\M)} \leq C\beta^{-1}$. Since we have that $\|v\|_{C^4(\M)} \leq C\beta^{-1}$ and $\|v\|_{C^2(\M)} \leq C$, we may use interpolation inequalities (see e.g. Lemma 6.32 in \cite{Gilbarg-Trudinger}) to deduce that $\|v\|_{C^3(\M)} \leq C \beta^{-1/2}$. These arguments then conclude the proof of Theorem \ref{thm:bias}.

 %We remark that there are a variety of other techniques for demonstrating regularity of elliptic equations; see for example the standard reference \cite{Gilbarg-Trudinger}.
\end{proof}

We remind the reader here that convergence of $v$ as $\beta \rightarrow 0$ is towards $\mu_f$, not $\mu$. One can only guarantee convergence towards $\mu$ if one makes more specific assumptions upon the label error distribution $p$ or on the empirical risk function $F$. We remark that the references in the previous proof referred to the Euclidean case, but can be extracted to the manifold case we consider here via standard localization arguments. \nc We also emphasize that there are many other techniques and technical challenges associated with elliptic regularity (especially associated with boundary values), which were not relevant in this context, a standard reference is \cite{Gilbarg-Trudinger}.

\bibliography{NGT-RM}
\bibliographystyle{siam}

\nc

\end{document}